\documentclass[acmlarge,screen,nonacm, authorversion]{acmart}
\usepackage{subcaption}
\usepackage[noabbrev]{cleveref}
\usepackage{amsmath}
\usepackage{algorithm}
\usepackage[noend]{algpseudocode}

\newtheorem{thm}{Theorem}
\newtheorem{lem}{Lemma}

\newtheorem{cor}{Corollary}
\newtheorem{defn}{Definition}[section]
\newtheorem{ass}{Assumption}

\DeclareMathOperator*{\argmin}{arg\!min}

\AtBeginDocument{%
  \providecommand\BibTeX{{%
    \normalfont B\kern-0.5em{\scshape i\kern-0.25em b}\kern-0.8em\TeX}}}

\setcopyright{none}

\begin{document}

\title{Intervention-Assisted Policy Gradient Methods for Online Stochastic Queuing Network Optimization: Technical Report}

\author{Jerrod Wigmore}
\email{jwigmore@mit.edu}
\affiliation{%
  \institution{Massachusetts Institute of Technology}
  \city{Cambridge}
  \state{Massachusetts}
  \country{USA}
}

\author{Brooke Shrader}
\affiliation{%
  \institution{MIT Lincoln Laboratory}
  \city{Lexington}
  \state{Massachusetts}
  \country{USA}
}

\author{Eytan Modiano}
\affiliation{%
  \institution{Massachusetts Institute of Technology}
  \city{Cambridge}
  \state{Massachusetts}
  \country{USA}
}

\renewcommand{\shortauthors}{Wigmore, Shrader, \& Modiano}

\begin{abstract}

\end{abstract}
\begin{abstract}

Deep Reinforcement Learning (DRL) offers a powerful approach to training neural network control policies for stochastic queuing networks (SQN). However, traditional DRL methods rely on offline simulations or static datasets, limiting their real-world application in SQN control.  This work proposes Online Deep Reinforcement Learning-based Controls (ODRLC) as an alternative, where an intelligent agent interacts directly with a real environment and learns an optimal control policy from these online interactions. SQNs present a challenge for ODRLC due to the unbounded nature of the queues within the network resulting in an unbounded state-space.  An unbounded state-space is particularly challenging for neural network policies as neural networks are notoriously poor at extrapolating to unseen states.  To address this challenge, we propose an intervention-assisted framework that leverages strategic interventions from known stable policies to ensure the queue sizes remain bounded.  This framework combines the learning power of neural networks with the guaranteed stability of classical control policies for SQNs.  We introduce a method to design these intervention-assisted policies to ensure strong stability of the network. Furthermore, we extend foundational DRL theorems for intervention-assisted policies and develop two practical algorithms specifically for ODRLC of SQNs. Finally, we demonstrate through experiments that our proposed algorithms outperform both classical control approaches and prior ODRLC algorithms.

\end{abstract}

\thanks{DISTRIBUTION STATEMENT A. Approved for public release. Distribution is unlimited.
This material is based upon work supported by the Department of the Air Force under Air Force Contract No. FA8702-15-D-0001. Any opinions, findings, conclusions or recommendations expressed in this material are those of the author(s) and do not necessarily reflect the views of the Department of the Air Force.

© 2024 Massachusetts Institute of Technology.

Delivered to the U.S. Government with Unlimited Rights, as defined in DFARS Part 252.227-7013 or 7014 (Feb 2014). Notwithstanding any copyright notice, U.S. Government rights in this work are defined by DFARS 252.227-7013 or DFARS 252.227-7014 as detailed above. Use of this work other than as specifically authorized by the U.S. Government may violate any copyrights that exist in this work.}
\maketitle

\section{Introduction}

In the field of Deep Reinforcement Learning (DRL), agents are often trained using offline simulated environments prior to being deployed on the real-world environment. DRL is a promising technique for training stochastic network control agents.  However, the traditional simulation-based training paradigm has two major pitfalls. If the true network dynamics are not able to be accurately captured in simulation, then the agent trained on the simulation dynamics may perform poorly on the real-network.  This is often referred to as the sim-to-real gap \cite{valassakis2020,salvato2021}. Additionally, the policies of agents are often overfit to the training environments, and thus struggle to generalize to unseen environments in their deployment. \cite{zhang2018a,farebrother2020,  zhang2018}. In the context of SQN control, an agent would have to be trained on all possible dynamics of a particular network if the true network dynamics are not known with certainty. To overcome these limitations, we propose an Online Deep Reinforcement Learning-based Controls (ODRLC) paradigm for training SQN control agents.  In ODRLC, an intelligent agent directly interacts with a real-world environment and learns to optimize its policy through these online interactions.  This approach ensures the agent's policy is optimized for the true environment and does not require access to simulations  prior to deployment.

However, applying ODRLC to SQN control tasks introduces additional challenges.  The infinite buffer model is widely used in the network control literature due to its analytical tractability, and as a reflection of the fact that network buffers are often exceedingly large in practice. This presents a challenge for utilizing neural networks (NNs) as policy or value function approximators as the state-space under the infinite buffer assumption is unbounded and NNs are poor at extrapolating or generalizing to unseen inputs. The combination of an unbounded state-space and NN dependent policy creates a catastrophic feedback loop in the ODRLC setting. When an agent encounters an unseen state, its will take a poor action due to the poor extrapolation capabilities of the neural network(s) of which its policy depend on.  This sub-optimal action drives the agent further into the unexplored region of the state-space causing a never-ending cycle.  We call this the \textit{extrapolation loop} of unbounded states-spaces. This feedback loop is mitigated in offline simulation based training by occasionally resetting the environment's state. However, state resets may incur a huge cost or be infeasible in ODRLC. This notion of ensuring the environment state remains within a finite region of the state-space is highly related to the notion of strong stability that is often desired for SQN control algorithms.

This work addresses the challenge of ODRLC for SQN control tasks with unbounded state spaces. We propose a novel intervention-assisted agent framework that leverages a known stable policy to guarantee network stability while incorporating a NN policy for exploration and policy improvement. We prove these intervention-assisted policies are strongly stable, enabling their use for ODRLC. We extend key DRL theorems to the intervention-assisted setting, and introduce two practical ODRLC algorithms.  Our experiments show that these algorithms outperform existing SQN control and DRL-based methods in the ODRLC setting.

\subsection{Related Works}
Despite the wide applicability of queuing network models to various domains such as communication networks, manufacturing, and transportation, and their rich historical context in the controls literature, the integration of DRL for SQN controls remains a relatively underexplored avenue. The authors of \cite{dai2022} leverage DRL to optmize for delay in SQN control tasks that are similar to those studied in this paper, however their methods are not developed for the ODRLC setting. In \cite{raeis2021}, the authors use Deep Deterministic Policy Gradient  to learn queuing network control policies via offline environments that provide explicit guarantees on the end-to-end delay of the policy.   
Each of these aforementioned works uses the standard offline simulation-based training paradigm of DRL and thus the algorithms do not extend well into the ODRLC setting.

The ODRLC setting is most similar to \textit{continuing} or \textit{average reward} Reinforcement Learning. In \cite{zhang2021}, the authors provides a novel policy improvement theorem for the average-reward case, which is fundamental in the development of trust-region methods including PPO.  Ma et. al propose a unified policy improvement theorem that combines both the average reward and discounted reward settings in addition to addressing the Average Value Constraint problem that arises in average reward DRL \cite{ma2021}. The theoretical results in both \cite{zhang2021,ma2021} hinge on the assumption that  the state-space is finite and thus don't apply to environments with unbounded state-spaces such as queuing networks. In \cite{pavse2023},  the authors develop a Lyapunov-inspired reward shaping approach that encourages agents to learn a stable policy for online DRL over unbounded state-spaces.
Safe-DRL is a branch of DRL which similarly leverages interventions during training. The goal in Safe-DRL is for an agent to maximize some reward function while also satisfying some safety constraints. The authors of \cite{wang2018} uses human interventions to aid in robotic navigation tasks.  Similarly, in \cite{wagener2021} automatic advantage-based interventions are used to enforce safety constraints while still using DRL algorithms designed for unconstrained tasks.  These interventions align with our strategy, emphasizing the importance of external guidance to ensure stability and safety in the training process.

\section{Preliminaries} \label{sec: prelim}
\subsection{Stochastic Queuing Network Model}

\begin{figure}
    \centering
    \includegraphics[width=0.6\linewidth]{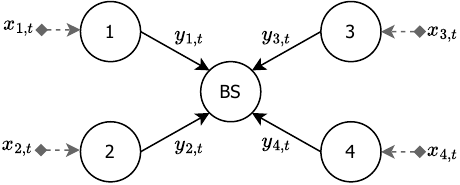}
    \caption{\textbf{(SH2)} An example of a single-hop wireless network. Packets arrive to each user according to user-dependent arrival distributions.  All packets are destined for the base-station. At each time step, the central controller chooses from one of the four links to activate.}
    \label{fig:SH2}
\end{figure}

\begin{figure}
    \includegraphics[width=0.6\linewidth]{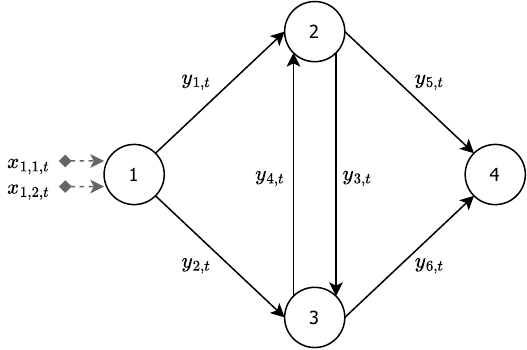}
    \caption{\textbf{(MH1)} An example of a multi-hop network.  Packets from two different classes arrive to node $1$ and all packets are destined for node $4$. At each time-step, the central network controller must choose an action $\mathbf a_t$ which dictates how many packets from each class is transmitted over each directed link.}
    \label{fig:MH1}
\end{figure}
In this paper, we focus on the objective of delay minimization for general discrete time SQNs with Markovian dynamics.  Under these models, the delay minimization task is well modeled by a Markov decision process.  The networks under consideration consist of nodes connected by directed links, with each node hosting one or more queues equipped with unbounded buffers that store undelivered packets.  Let $\mathbf q_t=\{q_{i,t}\}_{\forall i}$ denote the vector of all queue backlogs  within the network, and let $\bar q_t=\sum_{i}q_{i,t}$ denote the sum of backlogs across all queues within the network at the beginning of time $t$. The delay minimization objective is equivalent to minimizing the long-term average queue size \cite{little1961}. Thus our objective function can be defined as: 
\begin{equation}
    \min \lim_{T\rightarrow \infty} \frac{1}{T}\sum_{t=0}^{T-1}\bar q_t
\end{equation}

For the SQN models considered, two random processes -- stochastic packet arrivals and stochastic link capacities -- govern the dynamics, along with the actions taken by a central network controller or agent.   All packets belong to one of $K\geq 1$ traffic classes, where each class has an associated packet arrival distribution, arrival node, and a destination node.   For any class $k$,  $x_{k,t}$ packets  arrives at time $t$, where $x_{k,t}$ is drawn i.i.d. from a finite discrete distribution $\mathbb{P}(x_k)$.  The notation $\mathbf{x}_t=\{x_{k,t}\}_{\forall k}$ denotes the vector of all arrivals in time $t$.  

We denoted the capacity  of link $m$ at time step $t$ by $y_{m,t}$, which is also referred to as link $m$'s link state. At the beginning of each time step $y_{m,t}$ is sampled i.i.d. from a finite discrete distribution $\mathbb P(y_{m,t})$.  We use $\mathbf y_t=\{y_{m,t}\}_{\forall m}$ to denote the set of all link states over all $M$ links at time $t$.  We assume that arrival and link state distributions are mutually independent and independent of the overall network state. The network state at time $t$ is captured by $\mathbf{s}_t=(\mathbf{q}_t, \mathbf{y}_t)$, encompassing both the queue states $\mathbf q_t$ and link states $\mathbf y_t$.  At the start of each time step, the central network controller observes $\mathbf{s}_t$ and selects an action $\mathbf{a}_t$ from its policy $\pi$. The set of allowable actions depend on the specific network instance and its current state. An action $\mathbf a_t$ is a vector that specifies the amount of packets to transmit over each link for each class. The central controller aims to efficiently route all packets to their destinations by choosing which packets are transmitted on each link during each time step. A packet leaves the network once it arrives to its destination node. In \Cref{sec: experiments}, we test our algorithm on the following SQN control tasks: 
\subsubsection{Single-Hop Wireless Network Scheduling Task}
For single-hop wireless scheduling problems, the network instance is described by a set of $K$ user nodes, a base-station, and a single link between each user and the base-station.  There is a traffic class associated with each user, and the base-station serves as the destination node for all user's traffic. To model wireless interference constraints, only a single link may be activated by the central controller in each time step. When the central controller selects user $k$'s link at time step $t$, the number of packets that are transmitted to the base station is $a_{k,t}=\min\{q_{k,t}, y_{k,t}\}$. This constraint reflects that user $k$ can only transmit the number of packets in its queue $q_{k,t}$, and cannot exceed the link's capacity $y_{k,t}$. In the reinforcement learning setting, the central controller aims to learn a state-dependent scheduling policy $\pi$ that minimizes the long-term average backlog.

\subsubsection{Multi-hop Network Control Task}
For multi-hop networks, the network instance is described by a graph $\mathcal G(\mathcal V, \mathcal E)$ where $\mathcal V$ denote the set of nodes and $\mathcal E$ denotes the set of directed links between nodes. \Cref{fig:MH1} shows an example a topology of a multi-hop network.   Each of the $K$ packet classes have an fixed source node and destination node. Each node maintains $K$ queues, one for each class of traffic. At each time-step, the control policy observes the network state $\mathbf s_t = (\mathbf q_t, \mathbf y_t)$, and selects an action $\mathbf a_t = \{a_{m,k,t}\}_{\forall m,k}$ where  where $a_{m,k,t}\geq 0$ is the number of class $k$ packets to be transmitted on link $m$ in time-step $t$. This action $\mathbf a_t$ must satisfy the following constraint: 
\begin{align}
    \sum_{k} a_{m,k,t} &\leq y_{m,t}, \quad \forall m = 1,..., M \label{eq: link_constraint}
\end{align}
This link-capacity constraint means the total number of packets transmitted over each link must be less than the total capacity of the link. The central controller's decision $\mathbf a_t$ encompasses both a routing and scheduling decision.  It determines not only the path each packet takes but also the order at which each class of traffic is transmitted over each link in every time step.  Link activation constraints may also be included to model interference in wireless multi-hop networks, but we do not add this constraint for the experiments in \Cref{sec: experiments}

\subsection{Markov Decision Process}
Each SQN control tasks can be formulated as an average-cost Markov Decision Process (MDP) defined by the tuple $(\mathcal S, \mathcal A, c, P, \rho_0)$ where:
\begin{enumerate}
    \item $\mathcal S$ represents the state-space, comprised of all possible states $\mathbf s=(\mathbf q, \mathbf y)$.  As its assumed each buffer within the network is unbounded, the state-space $\mathcal S$ is also unbounded. 
    \item $\mathcal A$ denotes the action-space, comprised is the set of feasible control decisions which depends on the task. Additionally, we assume there is a set of valid actions $\mathcal A(\mathbf s)$ for each $\mathbf s\in \mathcal S$ that is known by the central controller for each SQN control task. 
    \item $P(\mathbf s'|\mathbf s, \mathbf a)$ is the probability of transitioning to state $\mathbf s'$ from state $\mathbf s$ after applying action $\mathbf a$. This transition probability captures the inherent uncertainty stemming from stochastic packet arrivals and stochastic link states.    
    \item  $c(\mathbf s_t)$ is the cost function. For delay minimization tasks, this equates to  $c(\mathbf s_t)=\bar q_t$.   
    \item $\rho_0(\mathbf s)=P(\mathbf s_0=\mathbf s)$ is the initial state distribution.  For \Cref{sec: experiments}, we assume that at the beginning of each task, all queues are empty and each link state is sampled randomly from its respective distribution.  However, none of our theoretical results depend on this assumption. 
\end{enumerate}

For an SQN control task, the central controller takes actions according to its policy $\pi$. We assume $\pi$ is stochastic and $\pi(\mathbf a|\mathbf s)$ denotes the probability of taking action $\mathbf a$ in state $\mathbf s$.  We use $\pi(\cdot|\mathbf s)$ to denote the distribution over all valid actions $\mathcal A(\mathbf s)$ in state $\mathbf s$.  In the ODRLC setting, $\pi$ generates a single long trajectory $\tau=(\mathbf s_0, \mathbf a_0,  c_0, \mathbf s_1, \mathbf a_1, ...)$ where $\mathbf s_0\sim \rho_0$, $\mathbf a_t\sim \pi(\cdot|\mathbf s)$, $c_t=c(\mathbf s_t)$, and $\mathbf s_{t+1}\sim P(\cdot|\mathbf s_t, \mathbf a_t)$. Unlike the traditional offline simulation setting, the state cannot be externally reset. However, the policy $\pi$ can be updated at fixed intervals of length $T_e$.   The aim is to learn a policy $\pi$ to solve the following average-cost minimization problem:
\begin{align} \label{eq:eta_pi}
    \min_{\pi\in \Pi} \eta(\pi) &=  \lim_{T\rightarrow \infty} \frac{1}{T}\sum_{t=0}^{\infty} \mathbb E_{\pi}[c(\mathbf s_t)] 
\end{align}
where $\mathbb E_{\pi}[\cdot]$ denotes the expectation under policy $\pi$. The policy space $\Pi$ denotes the set of all valid policies. We restrict  $\Pi$ to only include stationary Markovian polices. This means each $\pi\in \Pi$ makes decisions solely based on the current state, $\mathbf s_t$, and is independent of the time step $t$.

Any stationary Markov policy $\pi$ induces a Markov chain over the states with a state-transition distribution $P_\pi(\mathbf s'|\mathbf s)$.  When the state Markov Chain is positive recurrent we have the equivalence  $\eta(\pi)=\mathbb E_{\mathbf s\sim d(\pi)}[c(\mathbf s)]$ where  $d(\pi)$ is the steady-state distribution of the Markov chain induced by $\pi$. Note that for $\eta(\pi)$ to be finite, the state Markov chain must be positive recurrent.  When $\eta(\pi)$ is finite, the following value functions are well defined: 
\begin{align}
    V^\pi(\mathbf s) &= \mathbb{E}_\pi\left[\sum_{t=0}^{\infty}c(\mathbf s_t)-\eta(\pi)|\mathbf s_0= \mathbf  s\right] \\
    Q^\pi(\mathbf s,\mathbf a) &= \mathbb{E}_\pi\left[\sum_{t=0}^{\infty}c(\mathbf s_t)-\eta(\pi)|\mathbf s_0= \mathbf s, \mathbf a_0=\mathbf a\right] \\
   A^\pi(\mathbf s,\mathbf a) &= Q^\pi(\mathbf s,\mathbf a)-V^\pi(\mathbf s)
\end{align} 

\subsection{Lyapunov Stability}
Each queue backlog $q_{i,k}(t)$ must remain finite over any trajectory $\tau\sim \pi$, in order for the limit in \cref{eq:eta_pi} to be finite.  This requirement is strongly related to the following notion of stability:

\begin{defn}[Strong Stability \cite{neely2010}]
A discrete time process $\{q_t\}$ is strongly stable under transition function $P$ if for any initial state $q_0$ the following condition is satisfied:
\begin{equation}
\limsup_{T\rightarrow \infty}\frac{1}{T}\sum_{t=0}^{T-1}\mathbb E_{P}[|q_t|] < \infty    
\end{equation}
\end{defn}

If the queue state Markov chain $\{\mathbf q_t\}$ under $P_\pi$ is strongly stable for each queue in the network, the corresponding policy $\pi$ is called a strongly-stable policy. Strong stability is an important property for queuing networks that ensures the state Markov chain $\{\mathbf q_t\}$ is positive recurrent with unique steady-state distribution $d(\pi)$ that is independent of the initial state.  Additionally, strong-stability implies that the number of packets in each buffer remains finite which is essential in the ODRLC setting as it ensures finite packet delay throughout the learning process.   

Lyapunov Optimization is a technique for ensuring stability of dynamical systems through the use of Lyapunov functions. A Lyapunov function  $\Phi:\mathcal S\mapsto \mathbb R^+$ maps state vectors to non-negative scalars which quantify the ``energy" of each state.  Specifically, for SQNs, 
 $\Phi(\mathbf s_t)$ is typically defined to grow large as the queue sizes grow  large. Stability is achieved by taking actions that cause the Lyapunov drift defined as  $\mathbb E_\pi[\Delta(\mathbf s_t)] = \underset{\mathbf s_{t+1}\sim P_\pi(\cdot|\mathbf s)}{\mathbb E}[\Phi(\mathbf s_{t+1}) - \Phi(\mathbf s_t)|\mathbf s_t]$ to be negative when queue sizes grow too large.  For the SQN models considered in this work, the Lyapunov function is solely a function of the queue state $\Phi(\mathbf s_t) = \Phi(\mathbf q_t)$ as the link-states are not influenced by control decisions.  The following Lyapunov drift condition can be used to guarantee stability properties of classical network control algorithms:

\begin{thm} \label{thm: drift_stability}
The policy $\pi$ is strongly-stable if there exists a Lyapunov function $\Phi:\mathcal S\mapsto [0,\infty]$, a finite region of the state space $\mathcal S_1\subset \mathcal S$ and a finite constant $B$ such that:
\begin{align*}
\underset{\mathbf s_{t+1}\sim P_\pi(\cdot|\mathbf s_t)}{\mathbb E}[\Phi(\mathbf s_{t+1}) - \Phi(\mathbf s_t)|\mathbf s_t] \leq -(1+\bar q_t) + B \mathbf 1_{\mathcal S_{1}}(\mathbf s_t) \; \forall \mathbf s_t\in \mathcal S
\end{align*}
where, $$\mathbf 1_{\mathcal S}(\mathbf s_t)= \begin{cases}1, & \mathbf s_t\in \mathcal S_1, \\ 0, & \text{otherwise} \end{cases}$$
\end{thm}
This theorem is a modified version of the (V3) Lyapunov drift condition with $f=1+\bar q_t$ \cite[Ch. 14]{meyn2009}.  Examples of strongly-stable policies include the MaxWeight scheduling and Backpressure routing policies. \cite{meyn2009, tassiulas1992}. MaxWeight assigns weights ($q_{i,t}\times y_{i,t}$) to each node-link pair, and activates the link with the largest weight in each time step.  Backpressure dynamically routes traffic based on congestion gradients without prespecified paths in Multihop networks.  Both aim to to minimize bounds on the expected Lyapunov drift at each time step.  While the throughput benefits of these algorithms are well established, they may suffer from poor delay performance as shown  \Cref{sec: experiments}.
\subsection{Policy Gradient Methods}
In this work we focus on policy gradient methods, a class of DRL algorithms designed to directly optimize an agent's policy based on trajectories generated by the policy itself. We assume a NN is used to represent the policy for a particular task.  We refer to these policy-NNs as actor networks.  We denote a policy as $\pi_\theta$, where $\theta$ represents the weights of the actor network.   For parametric policies, the minimum cost objective can be expressed as $\min_{\theta\in \Theta} \eta(\pi_\theta)$.  This minimization is over the possible policy parameters $\Theta$, where $\Theta$ is determined by the actor network's architecture.   
Policy gradient methods perform this minimization iteratively: first estimating the gradient of $\eta(\pi_\theta)$ with respect to the actor network's parameters, $\theta$, and then performing gradient descent.  The analytical form of the gradient $\nabla_\theta \eta(\pi_\theta)$ is provided by the classical policy gradient theorem \cite{sutton1999}:
\begin{align}
    \nabla_\theta \eta(\pi_{\theta}) = \mathbb{E}_{\substack{\mathbf s\sim d({\pi_{\theta}})\\\mathbf a\sim \pi_{\theta}(\cdot|\mathbf{s})}}\left[Q^{\pi_{\theta}}(\mathbf s,\mathbf a)\nabla_\theta \log \pi_{\theta}(\mathbf a|\mathbf s)\right] \label{eq: pgrad}
\end{align}
where $d({\pi_{\theta}})$ is the stationary-distribution of the Markov chain induced by the policy $\pi_{\theta}$.  
Implementations that utilize automatic differentiation software work by constructing a loss function whose gradient approximates the analytical gradient:
\begin{equation} \label{eq: og_pg}
    L_{PG}(\pi_\theta) = \frac{1}{T}\sum_{t=0}^{T-1} \hat A^{\pi_\theta} (\mathbf s_t, \mathbf a_t)\log\pi_\theta(\mathbf a_t|\mathbf s_t)
\end{equation}
where $\hat A^{\pi_\theta}(\mathbf s_t,\mathbf a_t)$ is an estimate of the advantage function with respect to the policy $\pi_\theta$.  Utilizing the advantage function estimate is often preferred to an estimate of the state-action value $\hat Q^{\pi_\theta}(\mathbf s_t, \mathbf a_t)$ as it offers a lower-variance estimate of the gradient \cite{greensmith2004}.  A significant limitation of policy gradient algorithms is their one-time use of each state-action pair  $(\mathbf s_t,\mathbf a_t)\in \tau$.  Re-using trajectories for multiple gradient updates often leads to destructively large policy updates which cause ``performance collapse" \cite{schulman2017a}.

Trust region methods, a subset of policy gradient methods, are designed to address the challenge of making the largest possible steps towards performance improvements upon each update to the policy without risking performance collapse. The theoretical foundation for these methods are bounds on the relative performance  $\delta(\pi_\theta',\pi_\theta)=\eta(\pi_\theta')-\eta(\pi_\theta)$ between two policies:
\begin{align}
    &\delta(\pi_\theta',\pi_\theta) \leq \mathbb{E}_{\substack{\mathbf s\sim d({\pi_{\theta})}\\\mathbf a\sim \pi_{\theta}(\cdot|\mathbf s)}}\left[\frac{\pi_\theta'(\mathbf a|\mathbf s)}{\pi_\theta(\mathbf a|\mathbf s)}A^{\pi_{\theta}}(\mathbf s,\mathbf a)\right] + D(\pi_{\theta}', \pi_\theta) \label{eqn: improvbound}
\end{align}
where $D(\pi_{\theta}', \pi_\theta)$ is proportional to a measure of dissimilarity between the policies $\pi_\theta'$ and $\pi_\theta$. Practical algorithms define surrogate objectives that such that minimizing the surrogate objective corresponds to minimizing this upper bound. These surrogate objectives are typically defined as to minimize the first term while constraining the KL-divergence between $\pi_\theta'$ and $\pi_\theta$.  However, like policy-gradient methods, this optimization is performed via gradient updates on the policy parameters $\theta$, and thus effective algorithms must determine a suitable step-size to ensure $\delta(\pi_\theta',\pi_\theta)<0$.   Trust-region Policy Optimization (TRPO) approaches this by framing the problem as a constrained minimization, solved approximately via conjugate gradient methods \cite{schulman2017}. Proximal Policy Optimization (PPO), conversely, employs a clipped surrogate objective to deter substantial policy shifts, optimizing the bound while limiting large changes between successive policies \cite{schulman2017a}.

\section{Intervention-Assisted Policy Framework}
This section introduces the intervention-assisted policy framework for online training of stochastic queuing network control agents. This framework addresses two critical questions: (1) how to design a policy that guarantees strong stability in the ODRLC setting? (2) how to update this policy based on the online interactions with the SQN environment? The following assumption is required for this framework: 
\noindent \begin{ass} \label{ass: pi_0}
The agent has access to a known strongly stable policy $\pi_0$  
\end{ass}
\noindent This assumption is not restrictive for SQN controls. Classical SQN control algorithms such as the MaxWeight  or Backpressure policies can serve as $\pi_0$ for single-hop and multi-hop problems respectively \cite{meyn2009}. Its crucial to recognize that strong stability does not imply optimality, where the optimal policy is defined as $\pi^* = \argmin_{\pi}(\eta(\pi))$.  We restrict our attention to SQN environments that may be stabilized. In which case, the optimal policy $\pi^*$ is strongly stable.  

\subsection{Intervention Assisted Policy}
The intervention-assisted framework is based on the partitioning the state-space $\mathcal S$ into two disjoint regions: a bounded ``learning region" $\mathcal S_\theta$ and and unbounded ``intervention region" $\mathcal S_0$. When the current state $\mathbf s_t$ falls within the learning region $\mathcal S_\theta$, the agent samples an action $\mathbf a_t$ from the actor policy $\pi_\theta$. Conversely, if $\mathbf s_t\in \mathcal S_0$, the agent samples an action $\mathbf a_t$ from the known strongly stable policy $\pi_0$. Section \ref{sec: lr_selection} details a practical method of choosing this partitioning to ensure sample efficient learning. The intervention-assisted policy $\pi_I$ is formulated as follows: 
\begin{align} \label{eq: iapolicy}
&\pi_{I}(\cdot|\mathbf s) = I(\mathbf s)\pi_0(\cdot|\mathbf s) + (1-I(\mathbf s))\pi_\theta(\cdot|\mathbf s) \\ \intertext{where}
&I(\mathbf s) = \begin{cases}
    1, & \mathbf s\in \mathcal S_0,\\
    0, & \mathbf s\in \mathcal S_\theta
\end{cases}
\end{align}
We leverage on-policy policy gradient methods where the intervention policy $\pi_I$ is used by the agent to generate a trajectory $\tau\sim \pi_I$.  Each trajectory is a sequence of states, intervention indicators, actions, and costs: $\tau = (\mathbf s_0, I_0, \mathbf a_0, c_0, \mathbf s_1, ...)$ where $I_t\in\{0,1\}$ indicates if an intervention occurred at time-step $t$.  

\subsubsection{Guaranteeing Stability of Intervention-Assisted Policies}  
This section details how the intervention-assisted policy $\pi_I$ ensures strong stability using a Lyapunov optimization framework. Strong stability is vital from an SQN control perspective as it ensure packet delay is finite which is necessary for any policy deployed on a real-network.  From a learning perspective, strong-stability guarantees a steady-state distribution $d(\pi_I)$, which is is necessary for well-defined policy gradient  updates. The proofs of stability for intervention-assisted policies rely on the following lemma:

\begin{lem} \label{lem: finite_drift} Under the assumption that all arrivals are finite, if the Lyapunov function $\Phi(\cdot)$ is bounded for each $\mathbf s\in \mathcal S$,  then there exists a constant $B_i> 0$ for any bounded subset $\mathcal S_i\subset \mathcal S$,  such that
\begin{equation}
    \mathbb E_\pi[\Phi(\mathbf s_{t+1})-\Phi(\mathbf s_t)|\mathbf s_t\in \mathcal S_i] < B_i 
\end{equation}
for any policy $\pi$.  
\end{lem}
\Cref{lem: finite_drift} means that given a finite region $\mathcal S_i\subset \mathcal S$, the maximum conditional drift is bounded above.  Intuitively this is true because the max conditional drift is achieved by idling for any $\mathbf s_t$, and since arrivals are bounded, the max conditional drift is bounded. 

The following theorem provides details how to ensure an intervention-assisted policy $\pi_I$ is strongly stable.  
\begin{thm}\label{thm: pi_i}
Let $\mathcal S_\theta$ denote the learning region and $\mathcal S_0=\mathcal S\setminus \mathcal S_\theta$ denote the intervention region for an intervention assisted policy $\pi_I$.  If $\mathcal S_\theta$ is finite and $\pi_0$ satisfies \Cref{thm: drift_stability} for some $\Phi(\cdot)$, $B$, and $\mathcal S_1$, then the following Lyapunov drift condition is satisfied:
\begin{equation} \label{eq: drift_i}
    \mathbb E_{P_{\pi_I}}[\Phi(\mathbf s_{t+1})-\Phi(\mathbf s_t)|\mathbf s_t] \leq -(1+\bar q_t) + B_{1,\theta} \mathbf 1_{\mathcal S_{1,\theta}}(\mathbf s_t) \; \forall \mathbf s_t\in \mathcal S 
\end{equation}
for a constant $B_{1,\theta}<\infty$ and the region $\mathcal S_{1,\theta}=\mathcal S_1 \cup \mathcal S_\theta$ where $\cup$ denotes the union between two sets.     
\end{thm}
\begin{proof}[Proof of \Cref{thm: pi_i}]
The state space $\mathcal S$ can be dived into two  disjoint regions: $\mathcal S_1\cup \mathcal S_\theta=\mathcal S_{1,\theta}$, and $\mathcal S\setminus \mathcal S_{1,\theta} = \mathcal S_{1,\theta}^{-}$ . For any  $\mathbf s_t \in \mathcal S^{-}_{1,\theta}$, $I(\mathbf s_t)=1$  and the intervention policy $\pi_0$ is used.  Since $\pi_0$ is strongly stable, by \Cref{thm: drift_stability}:  \begin{equation} \label{eq: p1}
    \mathbb E_{P_0}[\Phi(\mathbf s_{t+1})-\Phi(\mathbf s_t)|\mathbf s_t\in \mathcal S_{1,\theta}^- ] \leq -(1+\bar q_t)
\end{equation}
The region $\mathcal S_{1,\theta}$ is finite thus by \cref{lem: finite_drift} there exists $B'_{1,\theta}<\infty$ such that:
\begin{equation}
    \mathbb E_{P_I}[\Phi(\mathbf s_{t+1})-\Phi(\mathbf s_t)|\mathbf s_t \in \mathcal S_{1,\theta}] < B'_{1,\theta}
\end{equation}
Finally $(1+\bar q_t)<\infty \; \forall \mathbf s_t\in \mathcal S$, thus there exists $B_{1,\theta}<\infty$ such that:

\begin{equation} \label{eq: p2}
    \mathbb E_{P_I}[\Phi(\mathbf s_{t+1})-\Phi(\mathbf s_t)|\mathbf s_t \in \mathcal S_{1,\theta}] < -(1+\bar q_t)+B_{1,\theta}
\end{equation}
\Cref{eq: drift_i} follows from combining \cref{eq: p1} and \cref{eq: p2}. 
\end{proof}
\begin{cor}
If the conditions of \Cref{thm: pi_i} are satisfied, the intervention-assisted policy $\pi_I$ is strongly-stable.
\end{cor}

Thus to achieve strong-stability given a strongly stable $\pi_0$, we only need to ensure that the learning region $\mathcal S_\theta$ is finite.  These conditions also mitigate the extrapolation burden on the actor network $\pi_\theta$. Since the possible state inputs into $\pi_\theta$ is confined to the finite region $\mathcal S_\theta$. Additionally, the conditions required for \Cref{thm: pi_i} are also independent of the actor policy $\pi_\theta$, meaning the intervention-assisted policies are agnostic to the parameterizations or the architecture used for the actor policy.  Now that the first question of stability has been addressed, the next section addresses the question of how to perform policy updates on an intervention-assisted policy. 

\subsection{Intervention Assisted Policy Gradients}
This section extends the classical policy gradient theorem \cite{sutton1999} to derive the analytical form of the gradient of the intervention-assisted performance objective $\nabla_\theta \eta(\pi_I)$.

\begin{thm}
    Given a strongly-stable intervention-assisted policy $\pi_I(\cdot|\mathbf s)=I(\mathbf s)\pi_0(\cdot|\mathbf s) + (1-I(\mathbf s))\pi_\theta(\cdot|\mathbf s)$, and average-cost objective $\eta(\pi_I)$, the policy gradient is:
    \begin{equation}
        \nabla_\theta \eta(\pi_I) = \underset{\substack{\mathbf s\sim d({\pi_I}) \\ \mathbf a\sim \pi_I(\cdot|\mathbf s)}}{\mathbb{E}}[(1-I(\mathbf s)Q^{\pi_I}(\mathbf s,\mathbf a)\nabla_\theta \log \pi_\theta(\mathbf a|\mathbf s)]\label{eq:iapg}
    \end{equation}
    where $d({\pi_I})$ is the steady-state distribution induced by $\pi_I$, and $Q^{\pi_I}$ is the state-action value function with respect to policy $\pi_I$. 
\end{thm}
\begin{proof}
We start by taking the partial derivative of the state value function $V^{\pi_I}(\mathbf s)$ with respect to our actor network parameters $\theta$. 
\begin{align*}
\frac{\partial V^{\pi_I}(\mathbf s)}{\partial \theta} &= \frac{\partial}{\partial\theta}\mathbb{E}_{\pi_I}\left[Q^{\pi_I}(\mathbf s,\mathbf a)|\mathbf a\sim\pi_I(\cdot|\mathbf s)\right] \\
&= \frac{\partial}{\partial\theta} \sum_{\mathbf a} \pi_I(\mathbf a|\mathbf s)Q^{\pi_\theta}(\mathbf s,\mathbf a) \\
\intertext{Using the definition of $\pi_I(\mathbf a|\mathbf s)$}
&= \frac{\partial}{\partial\theta} \sum_{\mathbf a} \left(I(\mathbf s)\pi_0(\mathbf a|\mathbf s) + (1-I(\mathbf s))\pi_\theta(\mathbf a|\mathbf s)\right)Q^{\pi_I}(\mathbf s,\mathbf a) \\
&=\sum_{\mathbf a}\left[\frac{\partial}{\partial\theta} I(\mathbf s)\pi_0(\mathbf a|\mathbf s) Q^{\pi_I}(\mathbf s,\mathbf a) + \frac{\partial}{\partial\theta}(1-I(\mathbf s))\pi_\theta(\mathbf a|\mathbf s)Q^{\pi_I}(\mathbf s,\mathbf a) \right]
\end{align*}
By the chain rule: 
\begin{align*}
    \frac{\partial}{\partial\theta} I(\mathbf s)\pi_0(\mathbf a|\mathbf s)Q^{\pi_I}(\mathbf s,\mathbf a) &= I(\mathbf s)\left[Q^{\pi_I}(\mathbf a,\mathbf s)\frac{\partial}{\partial\theta}\pi_0(\mathbf a|\mathbf s)+ \pi_0(\mathbf a|\mathbf s)\frac{\partial}{\partial\theta}Q^{\pi_I}(\mathbf s,\mathbf a)\right] \\
    &=I(\mathbf s) \pi_0(\mathbf a|\mathbf s)  \frac{\partial}{\partial\theta}Q^{\pi_I}(\mathbf s,\mathbf a)
\end{align*}
and similarly:
\begin{equation*}
    \frac{\partial}{\partial\theta} (1-I(\mathbf s))\pi_\theta(\mathbf a|\mathbf s)Q^{\pi_I}(\mathbf s,\mathbf a) = (1-I(\mathbf s))[Q^{\pi_I}(\mathbf s,\mathbf a) \frac{\partial}{\partial\theta}\pi_\theta(\mathbf a|\mathbf s) + \pi_\theta(\mathbf a|\mathbf s) \frac{\partial}{\partial\theta} Q^{\pi_I}(\mathbf s,\mathbf a) ]
\end{equation*}
The partial derivative of $Q^{\pi_I}(\mathbf s,\mathbf a)$ is:
\begin{align*}
    \frac{\partial}{\partial\theta} Q^{\pi_I}(\mathbf s,\mathbf a) &= \frac{\partial}{\partial\theta}\left[\mathbb{E}[c(\mathbf s)] -\eta(\pi_I) + \sum_{\mathbf s'}P(\mathbf s'|\mathbf s,\mathbf a)V^{\pi_I}(\mathbf s')\right] \\
    &= - \frac{\partial}{\partial\theta}\eta(\pi_I) + \sum_{\mathbf s'}P(\mathbf s'|\mathbf s,\mathbf a)\frac{\partial}{\partial\theta}V^{\pi_I}(\mathbf s')
\end{align*}
Using this fact, we get:
\begin{align*}
    \frac{\partial V^{\pi_I}(\mathbf s)}{\partial \theta} &= \sum_{\mathbf a}\left[I(s)\pi_0(\mathbf a|\mathbf s)\left(- \frac{\partial}{\partial\theta}\eta(\pi_I) + \sum_{\mathbf s'}P(\mathbf s'|\mathbf s,\mathbf a)\frac{\partial}{\partial\theta}V^{\pi_I}(\mathbf s')\right)\right] \\ &\quad+\sum_{\mathbf a}\left[Q^{\pi_I}(\mathbf s,\mathbf a)\frac{\partial}{\partial\theta}(1-I(\mathbf s))\pi_\theta(\mathbf a|\mathbf s) \right] \nonumber\\
&\quad+\sum_{\mathbf a}\left[(1-I(\mathbf s))\pi_\theta(\mathbf a|\mathbf s)\left(-\frac{\partial}{\partial\theta}\eta(\pi_I) + \sum_{s'}P(\mathbf s'|\mathbf s,\mathbf a)\frac{\partial}{\partial\theta}V^{\pi_I}(s')\right)  \right] \nonumber\\
&= -\frac{\partial}{\partial\theta}\eta(\pi_I)\sum_{\mathbf a} \left[I(\mathbf s)\pi_0(\mathbf a|\mathbf s) + (1-I(\mathbf s))\pi_\theta(\mathbf a|\mathbf s)\right] \\ &\quad \nonumber + \sum_{\mathbf a}\left[ \left(\pi_0(\mathbf a|\mathbf s) + (1-I(\mathbf s))\pi_\theta(\mathbf a|\mathbf s)\right)\sum_{\mathbf s'}\frac{\partial}{\partial\theta}P(\mathbf s'|\mathbf s,\mathbf a)V^{\pi_I}(\mathbf s')\right]\ \\
\nonumber & \quad +\sum_{\mathbf a}\left[ Q^{\pi_I}(\mathbf s,\mathbf a)(1-I(\mathbf s))\frac{\partial}{\partial\theta}\pi_\theta(\mathbf a|\mathbf s)\right]
\end{align*}
Using the fact that $\pi_I(\mathbf a|\mathbf s) = I(\mathbf s)\pi_0(\mathbf a|\mathbf s) + (1-I(\mathbf s))\pi_\theta(\mathbf a|\mathbf s)$ and $\sum_{\mathbf a} \pi_I(\mathbf a|\mathbf s)=1$:
\begin{align*}
    \frac{\partial V^{\pi_I}(\mathbf s)}{\partial \theta} &= -\frac{\partial}{\partial\theta}\eta(\pi_I) +\sum_{\mathbf a}\left[ \pi_I(\mathbf a|\mathbf s)\sum_{\mathbf s'}P(\mathbf s'|\mathbf s,\mathbf a)\frac{\partial}{\partial\theta}V^{\pi_I}(\mathbf s')\right]+\sum_{\mathbf a}\left[ Q^{\pi_I}(\mathbf s,\mathbf a)(1-I(\mathbf s))\frac{\partial}{\partial\theta}\pi_\theta(\mathbf a|\mathbf s)\right]
\end{align*}
Solving for $\frac{\partial}{\partial\theta}\eta(\pi_I)$ yields:
\begin{align*}
    \frac{\partial}{\partial\theta}\eta(\pi_I) &= \sum_{\mathbf a}\left[ \pi_I(\mathbf a|\mathbf s)\sum_{\mathbf s'}P(\mathbf s'|\mathbf s,\mathbf a)\frac{\partial}{\partial\theta}V^{\pi_I}(\mathbf s')\right]+\sum_{\mathbf a}\left[ Q^{\pi_I}(\mathbf s,\mathbf a)(1-I(\mathbf s))\frac{\partial}{\partial\theta}\pi_\theta(\mathbf a|\mathbf s)\right] -   \frac{\partial V^{\pi_I}(\mathbf s)}{\partial \theta}
\end{align*}
Summing both sides over the stationary distribution $d_{\pi_I}(\mathbf s)$ yields:
\begin{align*}
\sum_s d_{\pi_I}(\mathbf s) \frac{\partial}{\partial\theta}\eta(\pi_I) &= \sum_{\mathbf s} d_{\pi_I}(\mathbf s) \sum_{\mathbf a}\left[ \pi_I(\mathbf a|\mathbf s)\sum_{\mathbf s'}P(\mathbf s'|\mathbf s,\mathbf a)\frac{\partial}{\partial\theta}V^{\pi_I}(\mathbf s')\right]\\ \nonumber & \quad +\sum_{\mathbf s} d_{\pi_I}(\mathbf s)\sum_{\mathbf a}\left[ Q^{\pi_I}(\mathbf s,\mathbf a)(1-I(\mathbf s))\frac{\partial}{\partial\theta}\pi_\theta(\mathbf a|\mathbf s)\right] \\ \nonumber & \quad-   \sum_{\mathbf s} d_{\pi_I}(\mathbf s)\frac{\partial V^{\pi_I}(\mathbf s)}{\partial \theta}
\end{align*}
Since $d_{\pi_I}(\mathbf s)$ is stationary: 
\begin{align*}
\sum_{\mathbf s} d_{\pi_I}(\mathbf s)\sum_{\mathbf a} \pi_I(\mathbf a|\mathbf s)\sum_{\mathbf s'}P(\mathbf s'|\mathbf s,\mathbf a) = \sum_{\mathbf s} d_{\pi_I}(\mathbf s)
\end{align*}
As a result:
\begin{align*}
    \sum_{\mathbf s} d_{\pi_I}(\mathbf s) \frac{\partial}{\partial\theta}\eta(\pi_I) &= \sum_{\mathbf s} d_{\pi_I}(\mathbf s)\frac{\partial}{\partial\theta}V^{\pi_I}(\mathbf s')\\ \nonumber & \quad +\sum_{\mathbf s} d_{\pi_I}(\mathbf s)\sum_{\mathbf a}\left[ Q^{\pi_I}(\mathbf s,\mathbf a)(1-I(\mathbf s))\frac{\partial}{\partial\theta}\pi_\theta(\mathbf a|\mathbf s)\right] \\ \nonumber & \quad-   \sum_{\mathbf s} d_{\pi_I}(\mathbf s)\frac{\partial V^{\pi_I}(\mathbf s)}{\partial \theta} 
\end{align*}
The first and last sum cancels and $\eta(\pi_I)$ is not a function of $\mathbf s$, thus:
\begin{equation*}
    \frac{\partial}{\partial\theta}\eta(\pi_I)=\sum_{\mathbf s} d_{\pi_I}(\mathbf s)\sum_a Q^{\pi_I}(\mathbf s,\mathbf a)(1-I(\mathbf s))\frac{\partial}{\partial\theta}\pi_\theta(\mathbf a|\mathbf s)
\end{equation*}
\end{proof}

\Cref{eq:iapg} bears a strong resemblance to the original policy gradient theorem given in \cref{eq: og_pg} albeit with a few key distinctions. First, the expectation for the intervention-assisted policy gradient is with respect to the steady-state  and action distributions induced by the intervention-assisted policy $\pi_I$.  Additionally, the intervention-assisted policy gradient depends on the state-action value function $Q^{\pi_I}(\mathbf s,\mathbf a)$ which captures the state-action values with respect to the entire intervention-assisted policy instead of  just $\pi_\theta$ .  Like \cref{eq: pgrad}, the intervention-assisted policy gradient depends on $\nabla_\theta \log\pi_\theta(\mathbf a|\mathbf s)$, but note that the $(1-I(\mathbf s))$ term blocks direct contributions to the overall gradient from any states where an intervention occurred. The overall performance of $\pi_I$, including the contributions from $\pi_0$ during interventions, still effects the gradient through the state-action value function $Q^{\pi_I}(\mathbf s,\mathbf a)$ and the dependence on the steady-state distribution $d({\pi_I})$. Like in the non-intervention assisted case, \ref{eq:iapg} only theoretically supports a single update-step per trajectory generated. 

\subsection{Intervention-Assisted Policy Improvement Bounds}
This section establishes bounds of the form \cref{eqn: improvbound} for intervention-assisted policies.  These bounds allow us to extend the trust-region methods for intervention-assisted policies. 

\begin{thm} \label{thm: ia_pib}
Consider two different strongly stable intervention assisted policies $\pi_I'$ and $\pi_I$  that utilize the same learning region $\mathcal S_\theta$ and intervention policy $\pi_0$ and only differ in their actor policies $\pi_\theta'$ and $\pi_{\theta}$ respectively.  The performance difference is bounded as:
\begin{align} 
    \delta(\pi_I',\pi_I)   \leq &  \underset{\substack{\mathbf s\sim d({\pi_I})\\\mathbf a \sim \pi_\theta(\cdot|\mathbf s)}}{\mathbb E}\left[A^{\pi_I}(\mathbf s, \mathbf a)\mathcal R^{\theta'}_{\theta}(\mathbf a|\mathbf s) \mid \mathbf s\in \mathcal S_\theta\right] \  +  \underset{\substack{\mathbf s\sim d(\pi_I)\\\mathbf a \sim \pi_0(\cdot|\mathbf s)}}{\mathbb E}\left[A^{\pi_I}(\mathbf s, \mathbf a) \mid \mathbf s\in \mathcal S_0\right]  +  \mathcal D(d(\pi_I'), d(\pi_I)) \label{eq: thmbound}
\end{align}
\noindent where $R^{\theta'}_{\theta}(\mathbf a|\mathbf s) = \frac{\pi_{\theta'}(\mathbf a|\mathbf s)}{\pi_\theta(\mathbf a|\mathbf s)}$ and:
\begin{align*}
    \mathcal D(d({\pi_I'}),d({\pi_I})) &= 2 \sup_{\mathbf s\in \mathcal S} \left|\underset{\mathbf a \sim \pi_I'(\cdot|\mathbf s)}{\mathbb E}[A^{\pi_I}(\mathbf s, \mathbf a)]\right| D_{\mathrm{TV}}\left(d(\pi_I' \| d(\pi_I)\right)
\end{align*}
\end{thm}
\Cref{eq: thmbound} resembles analogous bounds for non-intervention-assisted methods shown in \cref{eqn: improvbound}, but with some key distinctions. The first difference is in conditioning. In \cref{eq: thmbound}, the first term is only considered when the state falls within the learning region $\mathcal S_\theta$. Similar to \cref{eq:iapg}, it depends on the intervention assisted policy through the advantage function $A^{\pi_I}(\mathbf s,\mathbf a)$ and on the actor policies through  $ R^{\theta'}_{\theta}(\mathbf a|\mathbf s)$. To minimize this term, the ratio $ R^{\theta'}_{\theta}(\mathbf a|\mathbf s)$ should be maximized (minimized) when $A^{\pi_I}(\mathbf s, \mathbf a)$ is negative (positive).  The next major difference is that \cref{eqn: improvbound} lacks the second term present in \cref{eq: thmbound}. This term accounts for the performance of the intervention assisted policy $\pi_I$ in the region $\mathcal S_0$.  This term goes to zero as the actor policy $\pi_\theta$ learns to keep the state within $\mathcal S_\theta$. The last term in \cref{eq: thmbound} is a measure of dissimilar between the steady-state distributions induced by $\pi_I'$ and $\pi_I$. This term is strictly positive, meaning to minimize the bound, the difference between policies should be minimized.   

The proof of \Cref{thm: ia_pib} requires the following Lemma, which is based on Lemma 2 in \cite{zhang2021} developed for the average-reward MDPs in bounded state-spaces.The proofs for the following lemma is identical to the proof found in \cite{zhang2021}, except that we must assume the two policies are strongly stable enable to ensures the existence of the steady state distributions for the case of unbounded state spaces.

\begin{lem}\label{lem: div}
For average cost MDPs with unbounded state-spaces, and strongly stable policies $\pi'$ and $\pi$ with corresponding steady-state distributions $d_{\pi'}$ and $d_{\pi}$, the following bounds hold:

\begin{align}
    \left|\eta\left(\pi^{\prime}\right)-\eta(\pi)-\underset{\substack{s \sim d_{\pi} \\
a \sim \pi^{\prime}}}{\mathbb{E}}\left[A^{\pi}(s, a)\right]\right|&\leq 
\sup_{\mathbf s\in \mathcal S} \left|\underset{\mathbf a \sim \pi'(\cdot|\mathbf s)}{\mathbb E}[A^\pi(\mathbf s, \mathbf a)]\right| \left|| d_{\pi'}-d_\pi \right||_1 \\ &
= 2 \sup_{\mathbf s\in \mathcal S} \left|\underset{\mathbf a \sim \pi'(\cdot|\mathbf s)}{\mathbb E}[A^\pi(\mathbf s, \mathbf a)]\right| D_{\mathrm{TV}}\left(d^{\pi^{\prime}} \| d^{\pi}\right)
\end{align}
where $D_{\mathrm{TV}}\left(d_{\pi^{\prime}} \| d_{\pi}\right)=\sup_{\mathbf s\in \mathcal S}|d_{\pi'}(\mathbf s)-d_{\pi}(\mathbf s)|$
\end{lem}

\begin{proof}[Proof of Theorem \ref{thm: ia_pib}]
From \Cref{lem: div}, we can bound the difference in average cost of strongly stable intervention policies $\pi_I'$ and $\pi_I$ as:
\begin{align*}
    \eta(\pi_I')-\eta(\pi_I) \leq \underset{\substack{s \sim d_{\pi} \\
a \sim \pi^{\prime}}}{\mathbb{E}}\left[A^{\pi_I}(s, a)\right ] + \mathcal D(d_{\pi_{\pi_I'}},d_{\pi_I})
\end{align*}
where $\mathcal D(d_{\pi_I'},d_{\pi_I}) = 2 \sup_{\mathbf s\in \mathcal S} \left|\underset{\mathbf a \sim \pi_I'(\cdot|\mathbf s)}{\mathbb E}[A^{\pi_I}(\mathbf s, \mathbf a)]\right| D_{\mathrm{TV}}\left(d_{\pi_I^{\prime}} \| d^{\pi_I}\right)$.

The expectation can be reexpressed using the importance sampling ration between policies:
\begin{align*}
    \underset{\substack{\mathbf s\sim d_{\pi_I}\\\mathbf a \sim \pi'_I(\cdot|\mathbf s)}}{\mathbb E}[A^{\pi_I}(\mathbf s, \mathbf a)] 
    & = \underset{\substack{\mathbf s\sim d_{\pi_I}\\\mathbf a \sim \pi_I(\cdot|\mathbf s)}}{\mathbb E}\left[\frac{\pi_I'(\mathbf a|\mathbf s)}{\pi_I(\mathbf a|\mathbf s)} A^{\pi_I}(\mathbf s, \mathbf a)\right]
\end{align*}

Using the law of total probability, we can break this expectation over the sum of expectations over the disjoint regions $\mathcal S_\theta$ and $\mathcal S_0$: 
\begin{equation}
    \underset{\substack{\mathbf s\sim d_{\pi}\\\mathbf a \sim \pi_I(\cdot|\mathbf s)}}{\mathbb E}\left[ A^{\pi_I}(\mathbf s, \mathbf a) \frac{\pi_I'(\mathbf a|\mathbf s)}{\pi_I(\mathbf a|\mathbf s)}\right]=
    \underset{\substack{\mathbf s\sim d_{\pi_I}\\\mathbf a \sim \pi_I(\cdot|\mathbf s)}}{\mathbb E}\left[ A^{\pi_I}(\mathbf s, \mathbf a) \frac{\pi_I'(\mathbf a|\mathbf s)}{\pi_I(\mathbf a|\mathbf s)}\bigg| \mathbf s\in \mathcal S_\theta\right] + \underset{\substack{\mathbf s\sim d_{\pi_I}\\\mathbf a \sim \pi_I(\cdot|\mathbf s)}}{\mathbb E}\left[ A^{\pi_I}(\mathbf s, \mathbf a)\frac{\pi_I'(\mathbf a|\mathbf s)}{\pi_I(\mathbf a|\mathbf s)}\bigg| \mathbf s\in \mathcal S_0\right]  
\end{equation}
Using the fact that $\pi_I(\mathbf a|\mathbf s,\mathbf s\in \mathcal S_\theta)=\pi_\theta(\mathbf a|\mathbf s, \mathbf s\in \mathcal S_\theta)$ and $\pi_I(\mathbf a|\mathbf s, \mathbf s\in \mathcal S_0)=\pi_0(\mathbf a|\mathbf s, \mathbf s\in \mathcal S_0)$:
\begin{align}
    \underset{\substack{\mathbf s\sim d_{\pi_I}\\\mathbf a \sim \pi_I(\cdot|\mathbf s)}}{\mathbb E}\left[ A^{\pi_I}(\mathbf s, \mathbf a) \frac{\pi_I'(\mathbf a|\mathbf s)}{\pi_I(\mathbf a|\mathbf s)}\bigg| \mathbf s\in \mathcal S_\theta\right] + \underset{\substack{\mathbf s\sim d_{\pi_I}\\\mathbf a \sim \pi_I(\cdot|\mathbf s)}}{\mathbb E}\left[ A^{\pi_I}(\mathbf s, \mathbf a)\frac{\pi_I'(\mathbf a|\mathbf s)}{\pi_I(\mathbf a|\mathbf s)}\bigg| \mathbf s\in \mathcal S_0\right] \\ = \underset{\substack{\mathbf s\sim d_{\pi_I}\\\mathbf a \sim \pi_\theta(\cdot|\mathbf s)}}{\mathbb E}\left[A^{\pi_I}(\mathbf s, \mathbf a)\frac{\pi_\theta'(\mathbf a|\mathbf s)}{\pi_\theta(\mathbf a|\mathbf s)} \bigg| \mathbf s\in \mathcal S_\theta\right] + \underset{\substack{\mathbf s\sim d_{\pi_I}\\\mathbf a \sim \pi_0(\cdot|\mathbf s)}}{\mathbb E}\left[A^{\pi_I}(\mathbf s, \mathbf a) \bigg| \mathbf s\in \mathcal S_0\right] 
\end{align}
where the importance sampling ratio disappears in the second term because $\frac{\pi_0(\mathbf a|\mathbf s)}{\pi_0(\mathbf a|\mathbf s)}=1$.  
\end{proof}

\section{Algorithms}
Building on the theoretical foundations of the previous sections, this section presents two practical algorithms for online training of intervention-assisted policies.  These algorithms follow the same structure of on-policy actor-critic reinforcement learning algorithms \cite{sutton2018}.  Both algorithms follow a two-phase approach, consisting of a policy rollout phase and policy update phase, repeated across multiple training episodes $e=1,2, ..., E$:
\begin{enumerate}
    \item \textbf{Policy Rollout Phase}: The current policy $\pi_I^{(e)}$ interacts with the environment and generates a trajectory $\tau^{(e)}$. 
    \item \textbf{Policy Update Phase} The trajectory $\tau^{(e)}$ is used in computing gradient updates. The trajectory is re-used in $U$ update epochs to provide a sequence of updated policies $\quad (\pi_I^{(e,0)}, \pi_I^{(e,1)}, ... \pi_I^{(e,U)})$ where $\pi_{I}^{(e,u)}$ for $u>0$ refers to the intervention assisted policy after the $u$th update epoch.
\end{enumerate}

Here $\pi_I^{(e,0)}$ corresponds to the original policy $\pi_I^{(e)}$ that generated the trajectory $\tau^{(e)}$ in policy rollout phase.  After all $U$ updates, the most recently updated policy $\pi_I^{(e,U)}$ becomes starting policy for the next episode $(e+1)$.

Each of the following algorithms differ only in their loss functions. For each algorithm, a trajectory $\tau^{(e)}$ is generated during the rollout phase of episode $e$, and the policy parameters are updated $U$ times via stochastic gradient descent:
\begin{align}
    \theta^{(e,u+1)} = \theta^{(e,u)}-\alpha \nabla_\theta \mathcal L_{pol}(\pi_I^{(e,u)},\tau^{(e)})
\end{align}
where $\alpha$ is the learning rate and  $L_{pol}(\pi_I^{(e,u)},\tau^{(e)})$ is the policy loss function which is a function of the current policy $\pi_I^{(e,u)}$ and the trajectory $\tau^{(e)}$. 

The first algorithm, the Intervention-Assisted Policy Gradient (IA-PG) algorithm, is an extension of the Vanilla Policy Gradient (VPG) algorithm\footnote{\href{https://spinningup.openai.com/en/latest/algorithms/vpg.html}{Vanilla Policy Gradient — Spinning Up documentation (openai.com)} } to the intervention-assisted setting. The IA-PG algorithm utilizes the following loss function:
\begin{align} 
   \!\!\!\!\mathcal L_{PG} \!\left(\pi^{(e,u)}_I, \tau^{(e)}\right) =     \frac{1}{T} \!\sum_{t=0}^{T-1}(1-I_t)\hat A^{\pi_I^{(e)}}_t\log \! \pi^{(e,u)}_{\theta}\!(\mathbf a_t|\mathbf s_t) \label{eq: ia_pg_loss}
\end{align}
where $A^{\pi_I^{(e)}}_t = A^{\pi_I^{(e)}}\!\!(\mathbf s_t, \mathbf a_t)$.
Similar to the non-intervention assisted case, this loss function is designed to have a gradient that approximates the analytical gradient presented in \cref{eq:iapg}.  However, to reduce variance, IA-PG employs the advantage function estimate $\hat A^{\pi_I^{(e)}}$  instead of the state-action value estimate $\hat Q^{\pi_I^{(e)}}$. It's important to note that this loss function does not incorporate any constraints on the policy updates. This lack of restriction can potentially lead to performance degradation issues.  

Our next algorithm, the Intervention-Assisted PPO (IA-PPO) algorithm, is designed to prevent such performance degradation. The IA-PPO algorithm builds upon the bound presented in  \cref{thm: ia_pib} to derive a loss function that resembles the original PPO algorithm \cite{schulman2017a}.  Recall that iteratively minimizing the the right-hand side of \cref{eq: thmbound} leads to a monotonically improving sequence of policies with respect to the average cost objective. Since only the actor network parameters $\theta$ change between updates, the bound in \cref{eq: thmbound} suggests solving the following optimization problem for updates:
\begin{align} \label{eq: tr_update}
    & \theta^{(e,u+1)} = 
    \underset{\theta}{\text{argmin}} \Bigg \{ \underset{\substack{\mathbf s\sim d(\pi_I^{(e)})\\ \mathbf a \sim \pi_I^{(e)}(\cdot|\mathbf s)}}{\mathbb E}\left [((1-I(\mathbf s))A^{\pi_I^{(e)}}\!\!(\mathbf s, \mathbf a)\mathcal \mathcal R^{(e,u)}_{(e)}(\mathbf a|\mathbf s)\right] + \mathcal D(d({\pi_I^{(e,u)}}), d({\pi_I^{(e)}})) \Bigg\}
\end{align}
\noindent where: $\mathcal R^{(e,u)}_{(e)}(\mathbf a|\mathbf s)=\frac{\pi_\theta^{(e,u)}(\mathbf a|\mathbf s)}{\pi_\theta^{(e)}(\mathbf a|\mathbf s)}$.

Notice, \cref{eq: tr_update} omits the second term in \cref{eq: thmbound} in as it does not depend on the variable $\theta$.  Minimizing the first term of \cref{eq: tr_update} encourages maximizing the ratio $\mathcal R^{(e,u)}_{(e)}(\mathbf a|\mathbf s)$ for negative advantages, and encourages minimizing $\mathcal R^{(e,u)}_{(e)}(\mathbf a|\mathbf s)$ for positive advantages.  However, the second term penalizes large deviations between policies.  As $\pi_0$ and $I(\mathbf s)$ remain unchanged between updates, the primary factor  influencing this term is the difference between the actor policies between updates. To this end, the IA-PPO algorithm uses the following clipped loss:
\begin{align} \label{eq: ppo_loss}
    \!&\mathcal L_{clip}(\pi_I^{(e,u)},\tau^{(e)})= \frac{1}{T}\sum_{t=0}^{T-1}(1-I(\mathbf s_t))  \max\{A^{\pi^{(e)}_{I}}_t \mathcal \mathcal R^{(e,u)}_{(e)}(\mathbf a_t|\mathbf s_t),\text{clip}(\epsilon, \hat A^{\pi_I^{(e)}}_t)
\end{align}
where $\epsilon\in (0,1)$ is a hyperparameter and
$$\text{clip} (\epsilon,A) = \begin{cases}
    (1+\epsilon) A, & A \geq 0 \\
    (1-\epsilon) A, & A < 0
\end{cases}$$
This clipped loss function creates more conservative updates by attempting to limit the divergence of policies between updates while still increasing (decreasing) the likelihood of actions that decrease (increase) the advantage. This focus on conservative updates is even more critical in online training compared to simulation-based training. Online training relies on a single sample path generated from the previous trajectory's end state, resulting in finite-length trajectories. This limitation leads to inherently noisier and potentially more biased advantage function estimates  compared to settings where multiple trajectories are generated from various starting states (simulation-based training). The clipped loss function helps to mitigate the impact of this noise and bias on policy updates.

\subsection{Critic Network and Advantage Estimation}
Both loss functions $\mathcal L_{PG}$ (\ref{eq: ia_pg_loss}) and $\mathcal L_{clip}$ (\ref{eq: ppo_loss}) depend on an estimate of the advantage function with respect to the intervention-assisted policy $\hat A^{\pi_I}(\mathbf s, \mathbf a)$. To obtain this estimate, our implementation utilizes a critic network in addition to an average cost variant of the Generalized Advantage Estimation (GAE) algorithm \cite{schulman2018}.  The critic network is a separate value NN trained to approximate the state-value function $V^{\pi_I}(\mathbf s)$. The critic estimator is notated as $V_\phi$ where $\phi$ notates the parameters of the critic network. To overcome the \textit{value drifting problem} that occurs in average reward/cost-based critic networks, we leverage the Average Value Constraint method as detailed in \cite{ma2021}.  For average cost problems, the critic network is updated via gradient descent to minimize the following loss function:
\begin{equation}
    \mathcal L_{val}(V_{\phi}, \hat V^{\pi_I}) = \frac{1}{T}\sum_{t=0}^{T-1}(\frac{1}{2}(V_{\phi(\mathbf s_t)}-\hat V(\mathbf s_t)+\nu b)^{2})
\end{equation}
where $\hat V^{\pi_I}(\mathbf s_t)$ is the target value for state $\mathbf s_t$, $b$ is an estimate of critic networks bias, and $\nu$ is a hyperparameters controlling the degree of bias correction. This average-cost critic is used in conjunction with an average cost-variant of the GAE algorithm to estimate $\hat A^{\pi_I}(\mathbf s_t, \mathbf a_t)$ for $(\mathbf a_t, \mathbf s_t)\in \tau$ where $\tau$ was generated using $\pi_I$. Details can also be found in \cite{ma2021}. 

A crucial distinction exists between training the actor $\pi_\theta$ and critic network $V_\phi$ for intervention-assisted policies. The policy loss functions $\mathcal L_{PG}$ and $\mathcal L_{clip}$ only directly utilize experience samples where interventions did not occur ($I_t=0$). In contrast, the critic loss function incorporates all experience samples. This is because the critic network estimates the state-value function of the intervention-assisted policy $V^{\pi_I}(\mathbf s)$ and not just the actor policy $V^{\pi_\theta}(\mathbf s)$. 

Algorithm 1 provides an outline of the IA-PG and IA-PPO algorithms as actor-critic style algorithms. The only difference between the two algorithm is the computation of $\mathcal L_{pol}$ in line 12, as the IA-PG algorithm uses \cref{eq: ia_pg_loss} while the IA-PPO uses \cref{eq: ppo_loss} for the policy loss.  The algorithm as written assumes that $\mathcal S_\theta$ and $\mathcal S_0$ have been pre-determined.  The next section details how these regions were selected for the results shown in \Cref{sec: experiments}.

\begin{algorithm}
\caption{Intervention-Assisted PG/PPO Algorithm}
\begin{algorithmic}[1]
\For{each epoch $e = 1, E$}
    \State \# \textit{Policy Rollout Phase}
    \State Initialize an empty trajectory buffer $\tau$
    \For{each step $t = 0,1 ..., T_e-1$}
        \State Observe state $\mathbf s_t$ and compute $I_t=\mathbf{1}(\mathbf s_t \in \mathcal S_0)$
        \State Sample action $\mathbf a_t\sim \pi_I(\cdot|\mathbf s_t)$
        \State Execute action $\mathbf a_t$, observe cost  $c_t$ and next state $\mathbf s_{t+1}$
        \State Store transition $(\mathbf s_t, I_t,,\mathbf a_t, c_t, \mathbf s_{t+1})$ in $\tau$
    \EndFor
    \State \# \textit{Update Phase}
    \State Estimate advantages $\hat{A}^{\pi_I^{(e)}}(s_t, a_t)$ $\forall$ $(a_t, s_t)\in \tau$ 
    \For{each update epoch $u = 1, U$}
        \State Compute policy loss $\mathcal L_{pol}$
        \State Compute value loss $\mathcal L_{val}$
        \State Update policy parameters: $\theta\gets \theta - \alpha \nabla_\theta L_{pol}$
        \State Update critic parameters function:  $\phi \gets \phi - \alpha \nabla_\phi L_{val}$
    \EndFor
\EndFor

\end{algorithmic}
\end{algorithm}

 \subsection{Learning Region Selection}\label{sec: lr_selection}
To achieve sample efficient learning, the finite learning region $\mathcal S_\theta$ should be specified to minimize the amount of interventions. To this end, we can leverage \Cref{thm: drift_stability} to ensure that interventions not only stabilize the network, but also push the network state back towards $\mathcal S_\theta$ in expectation.  Given a strongly-stable intervention policy $\pi_0$, according to \Cref{thm: drift_stability}, there exists a bounded sub-region $\mathcal S_1\in\mathcal S$ such that all states $\mathbf s\notin \mathcal S_1$ have negative expected drift, or more specifically:\begin{align}
    \mathbb E_{\pi_0}[\Phi(\mathbf s_{t+1})-\Phi(\mathbf s_t)\mid \mathbf s_t\notin \mathcal S_1] \leq -(\bar q_t+1)
\end{align}
Setting $\mathcal S_\theta=\mathcal S_1$ would ensure that each intervention results in negative expected drift, effectively pushing the state Markov chain back to $\mathcal S_\theta$ once it leaves.  If $\mathcal S_1$ is not known beforehand, it can be estimated by producing a trajectory using only $\pi_0$. In practice, it may be very difficult to estimate $\mathcal S_1$ exactly as it requires learning the relationship between high-dimensional state-space and the expected drift.  To address this challenge, we aim to learn a superset  $\mathcal S_g \supseteq \mathcal S_1$ where $\mathcal S_g$ can be estimated using a lower-dimensional representation of the states.   To this end, we use the following corollary:

\begin{cor} \label{cor: g_bound}
Given an strongly stable policy $\pi$ and a convex function $g:\mathcal S\mapsto [0,\infty)$, we can bound the expected  drift conditioned on $g(\mathbf s_t)$ $\forall \; \mathbf s_t$ as:
\begin{equation}
    \mathbb E_{P_\pi}[\Phi(\mathbf s_{t+1})-\Phi(\mathbf s_t)|g(\mathbf s_t)] \leq -(1+\bar q_t) + B_g\mathbf 1_{\mathcal S_g}(\mathbf s_t)
\end{equation}

where $\mathcal S_g = \{\mathbf s'\in \mathcal S: g(\mathbf s') \leq \max_{\mathbf s\in \mathcal S_1}g(\mathbf s)\}$ and $B_g$ is a constant.   
\end{cor}

This corollary ensures that if $\mathcal S_g$ is known, the expected drift for $\mathbf s \notin \mathcal S_g$ is negative. Letting $g(\mathbf s_t)=\bar q_t$ means $\mathcal S_g$ is defined based off the network backlog and we only need to estimate a $\bar q^{*}$ such that: 
\begin{equation}
    \mathbb E_{P_\pi}[\Phi(\mathbf s_{t+1})-\Phi(\mathbf s_t)|\bar q_t, \bar q_t > \bar q^{*}] \leq -(1+\bar q_t)
\end{equation}
 This quantity $\bar q_t^{*}$ is easier much easier to estimate compared to the exact region $\mathcal S_1$.  Once $\bar q_t^*$ is estimated, the intervention criteria can be defined as: 
 \begin{equation}
    I(\mathbf s_t) = \begin{cases}
        0, & \bar q_t \leq \bar q^* \\
        1, & \bar q_t > \bar q^*
    \end{cases}
\end{equation}
Note that under this criteria, $\mathcal S_\theta$ remains bounded thus the intervention assisted policy is strongly stable given that $\pi_0$ is strongly stable, and the expected drift given $I(\mathbf s_t)=1$ is negative. Note that $\bar q_t$ only contains partial information about the high-dimensional state $\mathbf s_t = (\mathbf q_t, \mathbf y_t)$ as it neglects all the information on the link states $\mathbf y_t$ in addition to averaging the information over the queue state $\mathbf q_t$.  As a result, $\mathcal S_\theta$ isn't minimal in the sense that it can contains some states such that $\mathbb E_{P_{\pi_0}}[\Phi(\mathbf s_{t_1})-\Phi(\mathbf s_t)|\mathbf s_t] <0$, but in practice we have found this backlog based intervention criteria a good strategy for sample efficient learning as long as we use a pessimistic estimate of $\bar q^*$. 

\subsubsection{Intervention threshold estimation} \label{sec: int_est}
To identify $\bar q^*$, we estimate   $\hat q^* = \min \{\bar q: \Delta_{\pi_0}(\bar q') < \omega, \forall \bar q' > \bar q \}$ where $\omega < 0$ is a negative constant that helps ensure our estimate is pessimistic in the sense that $\hat q^* > \bar q^*$. This process involves generating trajectory $\tau_0$ using only $\pi_0$,.  The length of this trajectory $T_0$ must be sufficiently long to characterize the relationship between $\bar q_t$ and $\Delta_{\pi_0}(\bar q_t)=\mathbb E_{\pi_0}[\Phi(\mathbf s_{t+1}) - \Phi(\mathbf s_t)|\bar q_t]$. We dynamically set $T_0$ to allow the state Markov chain to reach a steady-state distribution which is determined by the convergence of the time-averaged queue backlog. From the resulting trajectory, we compute $\Phi(\mathbf s_t)$ and $\delta(\mathbf s_t)=\Phi(\mathbf s_{t+1})-\Phi(\mathbf s_t)$ for $t=0,1, ..., T_0-1$.  Let $C(\bar q)$ be the number of times $\bar q_t=\bar q$, the point estimate of $\Delta_{\pi_0}(\bar q)$ is:
\begin{equation}
   \widehat{\Delta_{\pi_0}(\bar q)} =\frac{1}{C(\bar q)} \sum_{t=0}^{T_0-1} \delta(\mathbf s_t)\mathbf 1_{(\bar q_t = \bar q)}(\mathbf s_t)
\end{equation} 
This point-estimate may be noisy as $C(\bar q)$ is random, and backlogs with the highest counts will be aggregated around the average backlog over the trajectory.  As a result, the threshold estimator
\begin{equation}
    \hat q^* = \max\{\bar q: \widehat{\Delta_{\pi_0}(\bar q')} > \omega,; \forall,; \bar q' < \bar q  \}
\end{equation} 
may be overly pessimistic, as $\widehat{\Delta_{\pi_0}(\bar q)}$ for large $\bar q$ will be very noisy as the counts are low.  Note, that the counts are also low for very small $\bar q$ and the estimator will also have a high variance, but this has less of an effect on estimating $\bar q^*$ because we are taking an argmax. To improve the estimator $\hat q^*$, we drop the largest $5\%$ of estimates $\widehat{\Delta_{\pi_0}(\bar q)}$ , and then run a weighted moving average of length $10$ over the remaining point estimates with the counts corresponding to the weights.  \Cref{fig:SA4_drift0} shows an example of the point vs the weighted estimates of $\widehat{\Delta_{\pi_0}(\bar q)}$ for the SH2 network example using $\omega=-0.1$.  The backlog intervention threshold is estimated as  $\hat q^*_{weighted}= 22$ using the weighted estimator, and $\hat q_{point}^{*}=44$ using the point estimator.  This would have a negative impact on the sample efficiency of our intervention-assisted methods, as the region $\mathcal S_\theta$ would be much larger than $\mathcal S_1$. 
\begin{figure}[H]
    \centering
    \includegraphics[width=0.5\linewidth]{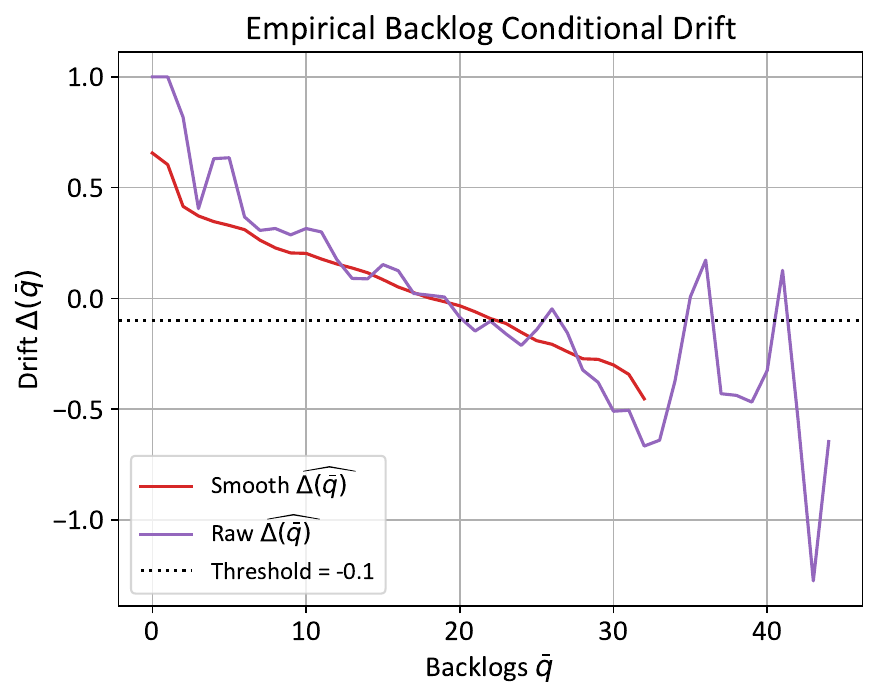}
    \caption{The estimated drift $\widehat{\Delta(\cdot)}$as a function of the queue backlog $\bar q$ for the MaxWeight policy on the SH2 network. The `Smooth' estimate uses a weighted moving average while the `Raw' estimate does not. The dotted line corresponds to an intended minimum drift of $\omega=-0.1$. }
    \label{fig:SA4_drift0}
\end{figure}
Once $\hat q^*$ is obtained, then the full intervention-assisted policy $\pi_I$ can be used to generate all future trajectories. 
\subsubsection{Updating Intervention Thresholds}
The initial estimate $\hat q^*$ is usually sufficient to allow for the optimized performance of the intervention-assisted policy to outperform the measured baselines.  However, it may be desirable to increase $\hat q^*$ as the agent learns an improved policy such that the agent no longer needs interventions.  However, increasing $\hat q^*$ linearly with $t$ may be a problem as increasing $\mathcal S_\theta$ without appropriately learning how to take optimal decisions near the boundary beforehand can lead to the instability feedback loop as $\bar q_t$  ``chases" a linearly increasing $\hat q^*$.  This this end, we use the following update formula to update $\hat q^{*}_k$after each trajectory $\tau_k \sim \pi_I^{(k)}$
\begin{equation}
    \hat q_{k+1}^* = \hat q_k^* + \gamma (1-R(\tau_k))\mathbf 1_{(R(\tau_k) > R_{min})} 
\end{equation}
 Here $\gamma$ is a hyperparameter that controls the magnitude of the updates between each trajectory, $R(\tau_k)=\frac{1}{T}\sum_{t=0}^{T-1}I_t$ is the intervention rate over trajectory $\tau_k$, and $\mathbf 1_{R(\tau_k) > R_{min}}$ is an indicator function that evaluates to $1$ when the intervention rate is greater than some hyperparameter $R_{min}\in [0,1]$. The indicator function is included to prevent $\hat q_{k+1}^*$ from growing arbitrarily large when no interventions are occurring as this corresponds the agent not visiting states near the boundary of $\mathcal S_\theta$.

\section{Experiments} \label{sec: experiments}
We conducted a series of experiments to evaluate the IA-PG and IA-PPO algorithms. The following SQN environments were used in the experiments:
\begin{enumerate}
    \item \textbf{SH1}: A two user $(K=2)$ single-hop wireless network.   
    \item \textbf{SH2}: A four user $(K=4)$ single-hop wireless network. The topology is shown in \Cref{fig:SH2}.
    \item \textbf{MH1}: A multihop environment with two classes $(K=2)$, six links $(M=6)$, and four nodes $(N=4)$.  The topology is shown in \Cref{fig:MH1}. 
    \item \textbf{MH2}: A multihop environment with four classes $(K=4)$, thirteen links $(M=13)$,  and eight nodes $(N=8)$. The topology is shown in \Cref{fig:MH2}.  
\end{enumerate}
\begin{figure}
    \centering
    \includegraphics[width=0.6\linewidth]{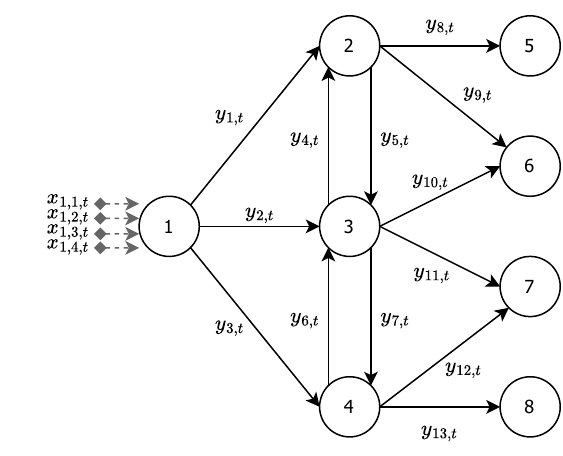}
    \caption{\textbf{(MH2)}: Node 1 is the arrival node for all four classes of packets. Nodes 5,6,7,8 are the destination nodes for classes 1,2,3, and 4 respectively. }
    \label{fig:MH2}
    \vspace{-10pt}
\end{figure}
The arrival and service distributions for each SQN environment are found in the Appendix. 

\subsection{Comparison Algorithms}

We evaluate the performance of all learning algorithms against the MaxWeight algorithm for single-hop network instances and Backpressure for the multi-hop network instances.  In addition to these classic network control algorithms, we evaluate the performance of the following DRL algorithms developed for average-reward tasks: 
\begin{enumerate}
    \item \textbf{Average Cost PPO (AC-PPO)}\cite{ma2021}: an average-cost variant of the original PPO algorithm that does not leverage interventions.
    \item \textbf{Stability then Optimality PPO (STOP-PPO)} \cite{pavse2023}: an average reward policy gradient algorithm designed for environments with unbounded state-spaces.  STOP-PPO utilizes reward shaping to first train the agent to learn how to stabilize the queuing network before learning how to optimize the queuing network.  Our variant differs from the original as it includes the PPO clipping mechanism in the policy loss function and utilizes the Average Value Constraint method to control the bias of the critic network.  
\end{enumerate}

\subsection{ODRLC Experiment Procedure}

The following online-training process akin to an ODRLC setting for all algorithms (IA-PG, IA-PPO, AC-PPO, and STOP-PPO). The agent interacts continuously with the SQN environment from $t=0$ until a long-time horizon $T_{end}$.  The performance of the agent is monitored over the entire long trajectory. We measure the following two metrics: the time-averaged backlog $\bar q_t^{(t)}=\frac{1}{t}\sum_{h=0}^{t-1}\bar q_h$ and $T_{MA}=10,000$ step moving average $q_t^{(MA)}=\frac{1}{T_{MA}} \sum_{h=t-T_{MA}}^{t-1}\bar q_t$. The moving average captured shorter-term performance, while the time-averaged metric assessed performance up to the current time step.

The experiment time horizon $T_{end}$ was divided into distinct episodes of length $T_e$. For the IA-PG and IA-PPO algorithms, the first $E_0$ episodes only the intervention policy $\pi_0$ was used. These trajectories $(\tau^{(0)}, \tau^{(1)}, ... \tau^{(E_O})$ were then used to estimate $\bar q_t^{*}$ and determine the learning and intervention regions $\mathcal S_\theta$ $\mathcal S_0$ according to the process outlined in \ref{sec: lr_selection}.  The performance of $\pi_0$ was measured and included in analysis of the intervention-assisted algorithm's performance.  After episode $E_0$, the full intervention-assisted policy $\pi_I^{(e)}$ is used to generate all future trajectories.  After each trajectory $\tau^{(e)}\sim \pi_I^{(e)}$ was generated, the actor network was updated $U$ times using the corresponding policy loss function.  The AC-PPO and STOP PPO algorithms the same training procedure, minus the initial learning region estimation phase meaning their actor policy $\pi_\theta$ generates all trajectories starting from $t=0$.  For all algorithms, the environment state is never reset. Additionally, for the IA-PG, IA-PPO, and AC-PPO algorithms, the cost shaping function $r'(\mathbf s_t)=\frac{-1}{1+\bar q_t}$ was used.  This cost shaping function ensures that the scales of costs are similar for different environments even if the backlog of the the respective optimal policies differ substantially, which allowed us to use the same learning rate for all environments as the magnitude of the gradients were comparable. We also used the symmetric natural log state transformation for all DRL algorithms to decrease the magnitude of divergence between inputs to the actor and critic networks \cite{pavse2023}.

All experiments were repeated five times for each algorithm using the same random seeds. This ensured identical arrival processes and link states across corresponding algorithms in each environment. All algorithms employed the Average Value Constrained Critic, with advantages estimated using an average-cost variant of the Generalized Advantage Estimation (GAE) algorithm. For consistency, identical hyperparameters were used across all environments for each algorithm if they shared hyperparameters.A detailed description of hyperparameters and network architectures can be found in the appendix.
\begin{figure*}[!htb]
    \centering
    \begin{subfigure}{\textwidth}
        \centering
        \includegraphics[width = 0.4\textwidth]{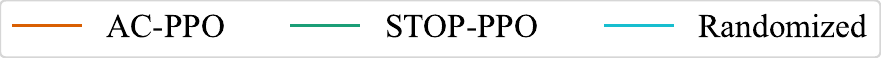}
    \end{subfigure}
    \begin{subfigure}[t]{0.45\textwidth}
        \centering
        \includegraphics[width = \textwidth]{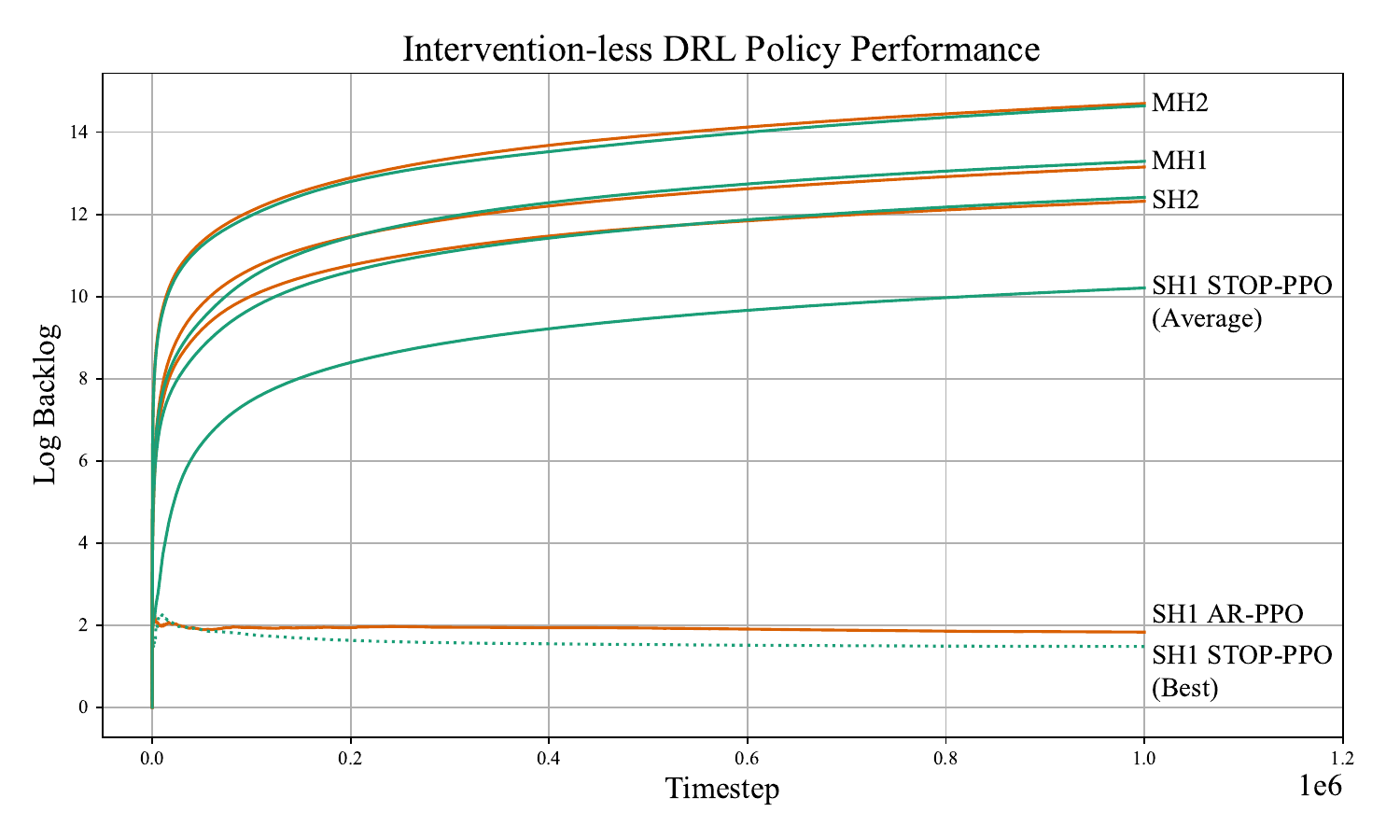}

    \end{subfigure}
    \begin{subfigure}[t]{0.45\textwidth}
        \centering
        \includegraphics[width=\textwidth]{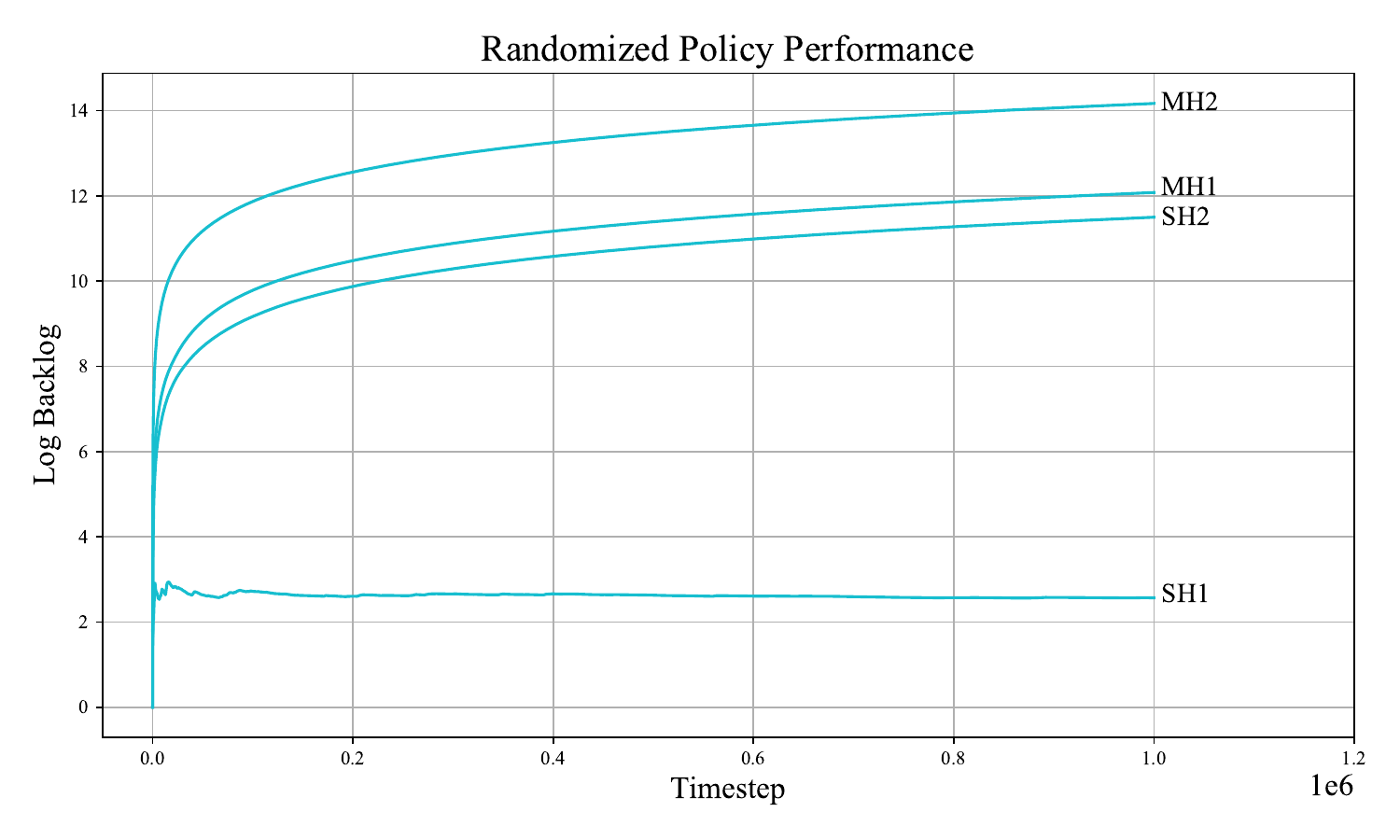}

    \end{subfigure}\vspace{-10pt}
\caption{Performance of the AC-PPO, STOP-PPO, and randomized policy on each network scenario.  The Y-axis represents the natural logarithm of the time-averaged backlog $\log(\bar q_t^{(t)})$. Each solid line represents an average over five seeds.} \label{fig: ac_stop_rand}
\end{figure*}
\subsection{Results}
\subsubsection{Intervention-less DRL Baselines}
We start by demonstrating how the intervention-less DRL algorithms struggle to stabilize the queues resulting in very poor performance on most environments.  The online-performance of the AC-PPO and STOP-PPO algorithms are shown in \Cref{fig: ac_stop_rand}.  For the SH2, MH1, and MH2 environments, neither the AC-PPO nor the STOP-PPO algorithm can can stabilize the queuing network resulting in the networks queue backlog growing without bounds.  For the SH1 environment, the AC-PPO algorithm was able to stabilize the queuing network for each seed while the STOP-PPO algorithm only stabilized the queuing network in three of the five seeds.   The variation in performance  between the SH1 network and the other network scenarios is best explained by examining the performance of a randomized policy on each network scenario.  Only in the SH1 network scenario does the randomized policy stabilize the queuing network.  The randomized policy performance is a good indicator of whether or not an intervention-less policy can work in the ODRLC setting as an untrained agent's initial policy is typically close to a randomized policy.  If this initial randomized policy is stable and the policy updates are conservative enough, as enabled by PPO-style updates, then its possible for the agent to avoid the extrapolation loop. However, its evident that an intervention-less approach to ODRLC will fail on many SQN control tasks.

\begin{figure*}[!htb]
    \centering

    \begin{subfigure}{\textwidth}
        \vspace{-6cm}
        \centering
        \includegraphics[width = 0.75\textwidth]{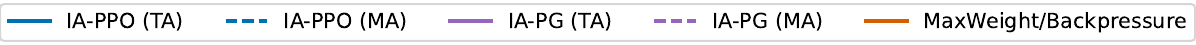}
    \end{subfigure}
    
    \begin{subfigure}{0.49\textwidth}
        \centering
        \includegraphics[width = \textwidth]{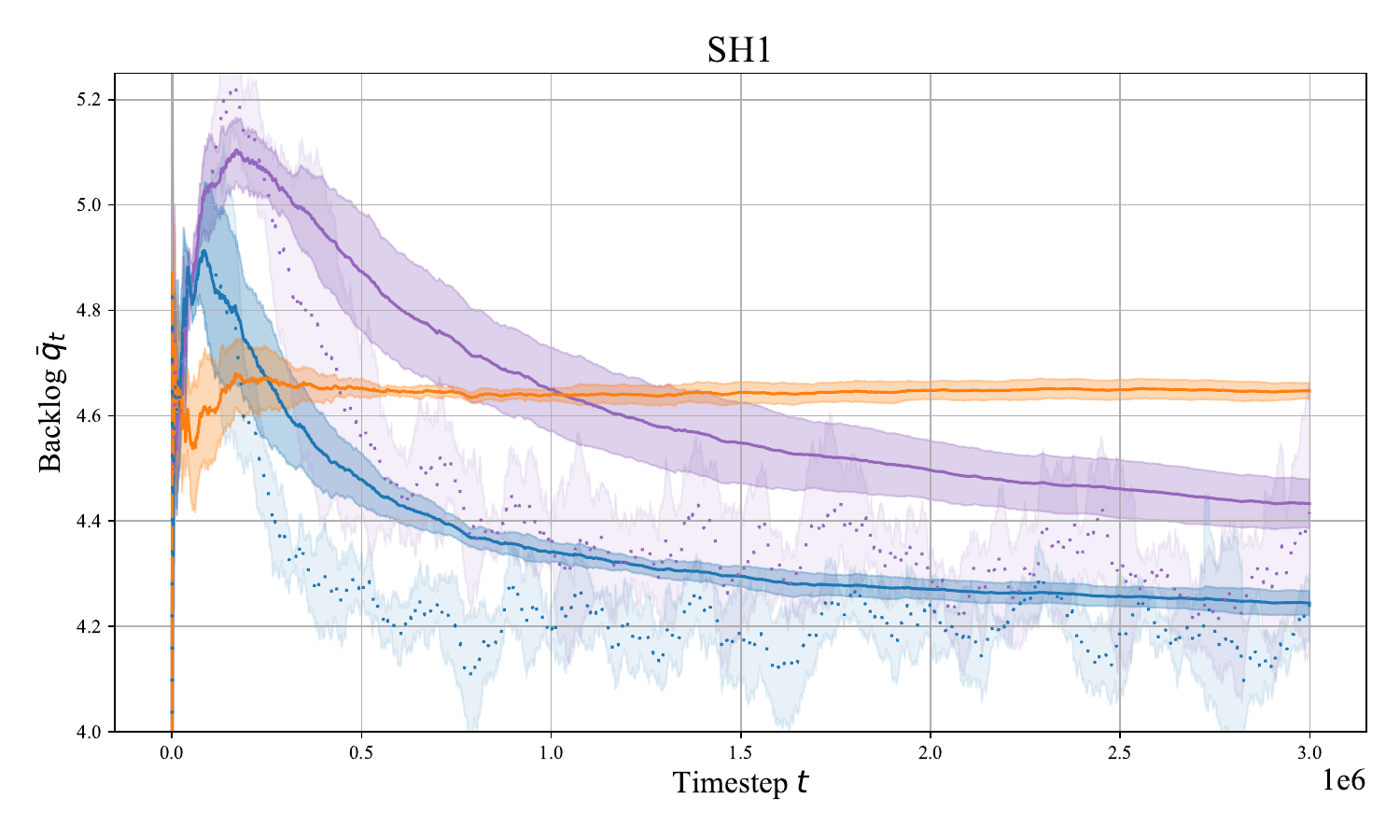}
    \end{subfigure}
    \begin{subfigure}{0.49\textwidth}
        \centering
        \includegraphics[width=\textwidth]{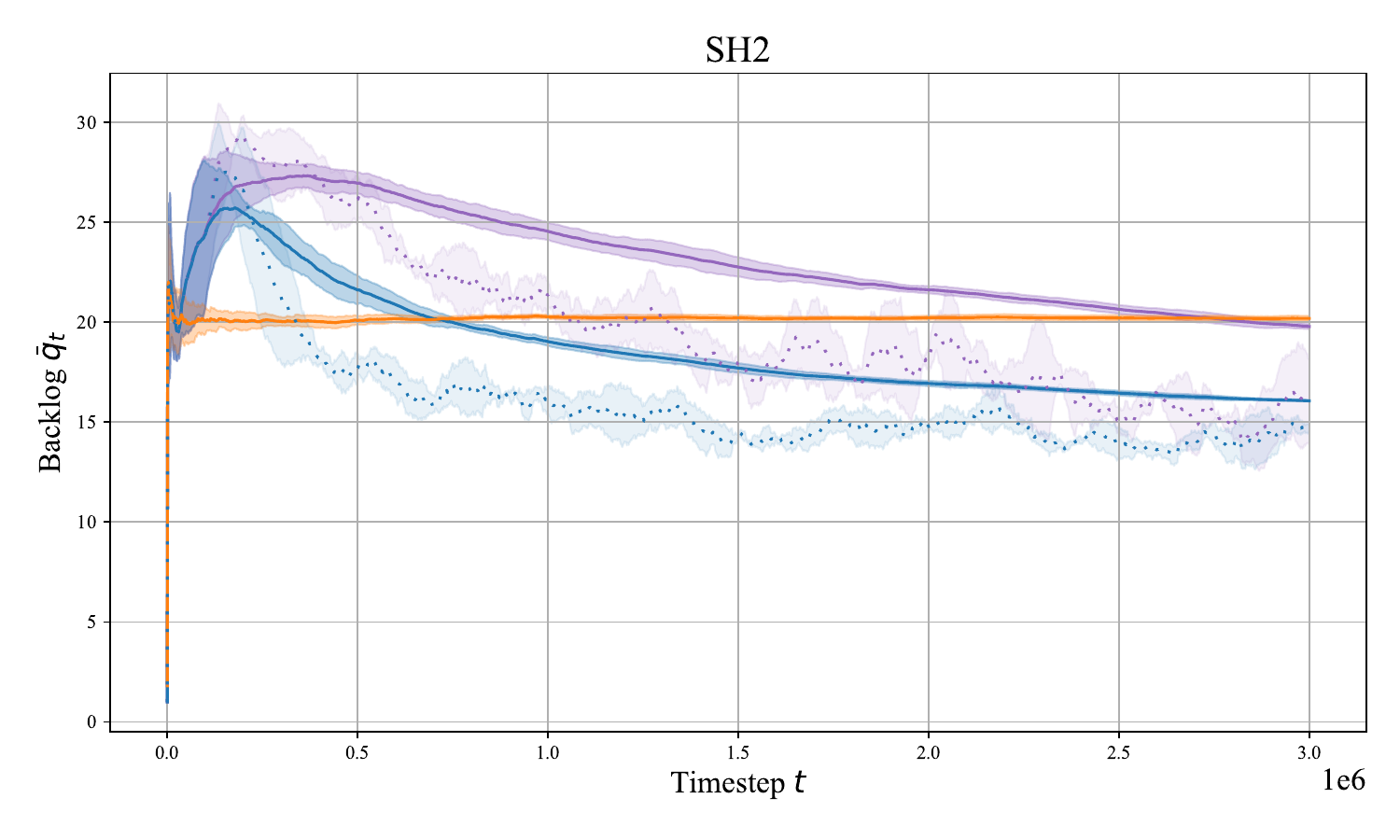}
    \end{subfigure}
    \begin{subfigure}{0.49\textwidth}
        \centering
        \includegraphics[width=\textwidth]{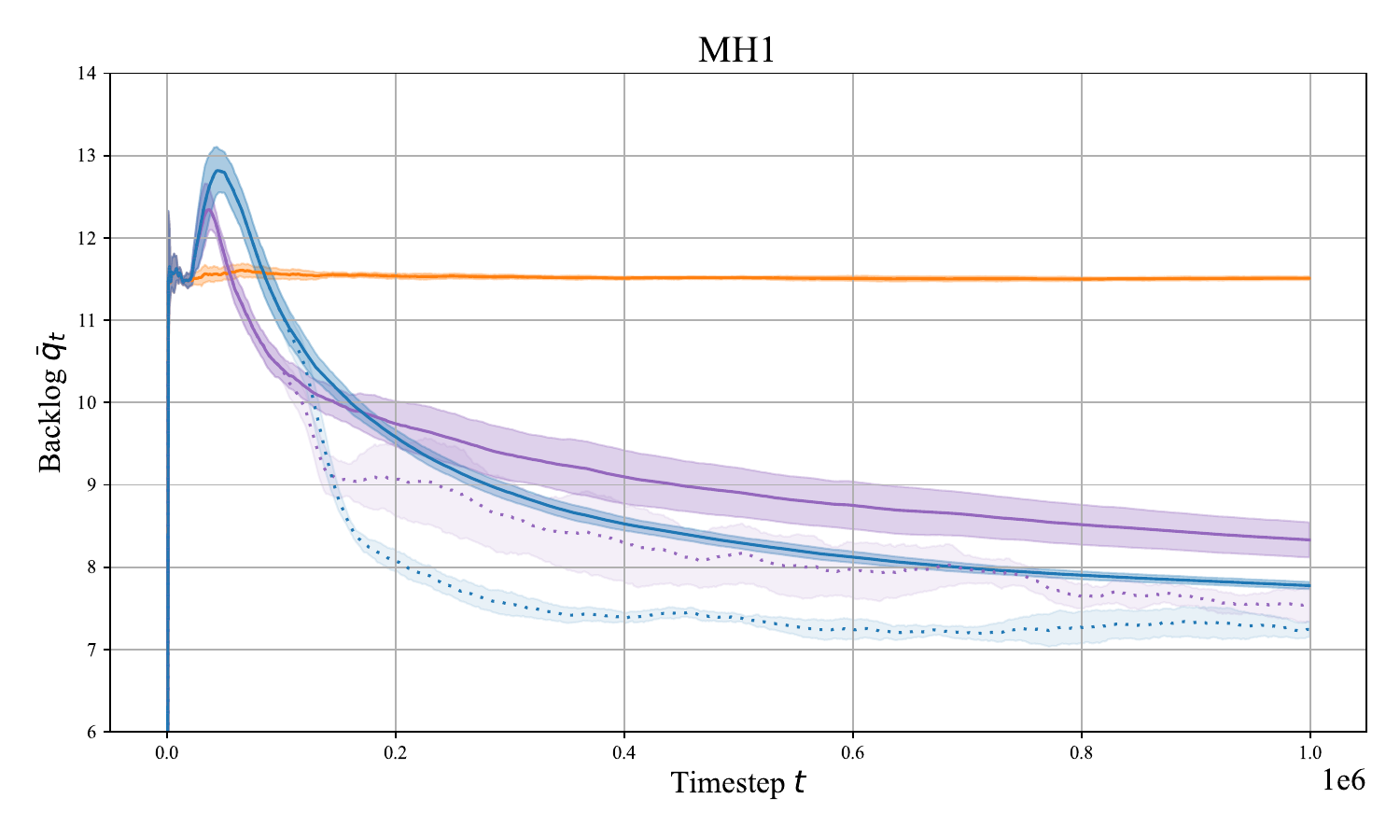}
    \end{subfigure}
    \begin{subfigure}{0.49\textwidth}
        \centering
        \includegraphics[width=\textwidth]{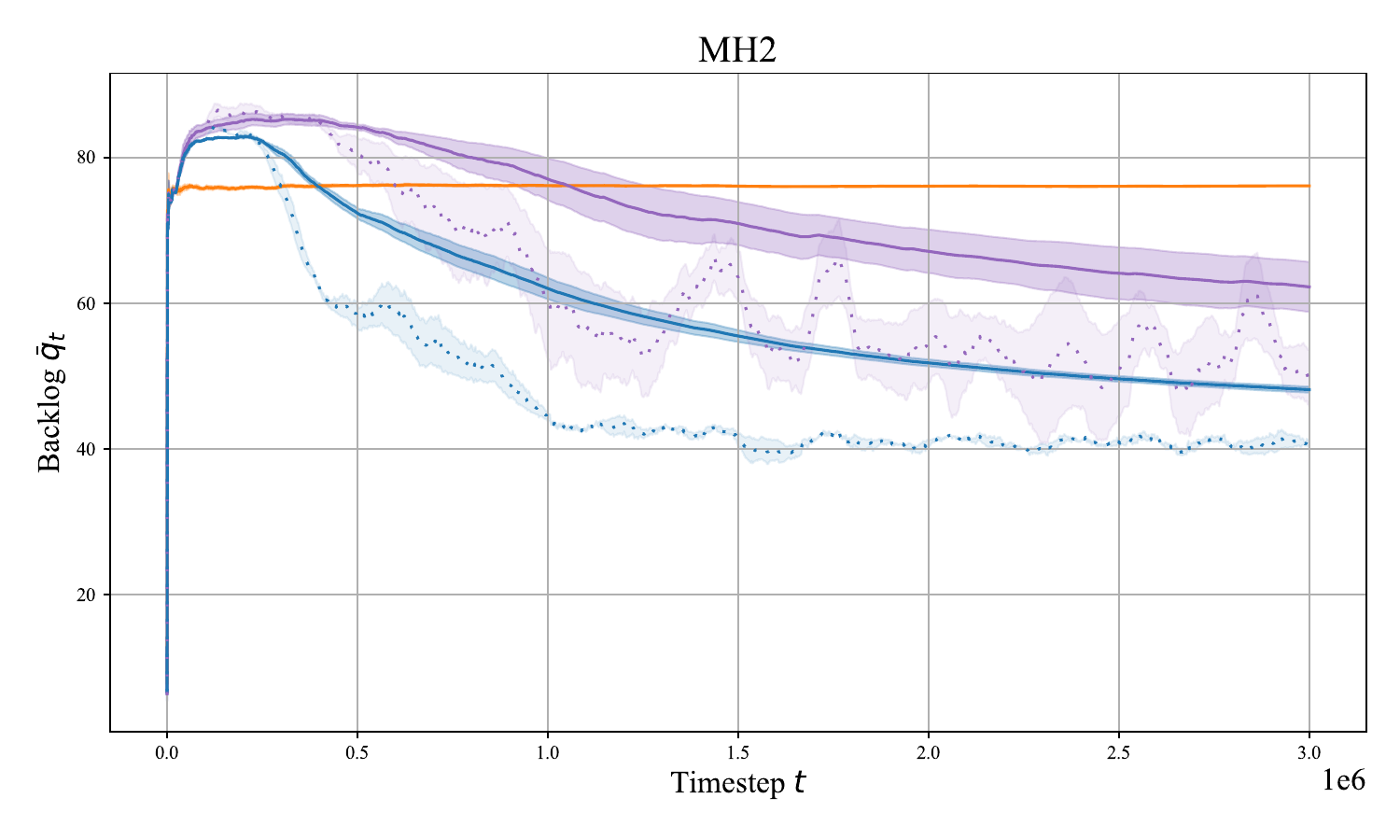}
    \end{subfigure}
    \vspace{-10pt}
    \caption{Performance of the IA-PG, IA-PPO, and MaxWeight/Backpressure algorithms on each environment.  The Y-axis represents the actual queue backlog $\bar q_t$.  Each line represents an average over five seeds. The solid lines correspond with the time-averaged backlog metrics $\bar q_t^{(t)}$ and the dashed lines correspond with a $T_{MA}=10,000$ step moving average $\bar q_t^{(T_{MA})}$. We report the $95\%$ confidence intervals for all performance metrics.} \label{fig: IA_results}
\end{figure*}
\subsubsection{Intervention-Assisted Algorithms}
Now that we have established the necessity of intervention-assisted methods for online-training of queueing network control algorithms we demonstrate that the IA-PG and IA-PPO algorithms can learn a better policy than classical network control algorithms online.   These results are shown in \Cref{fig: IA_results}.  Since the intervention-less DRL approaches failed in a majority of the environments, we focus on the comparison between the intervention-assisted algorithms and the classical network control algorithms as a baseline.   In all environments, the time-averaged backlog of the intervention-assisted algorithms outperforms the non-learning baseline. It is also evident that the IA-PPO algorithm is more sample-efficient than the  IA-PG algorithm.  This is best seen by the average rate at which $\bar q_t^{(T_{MA})}$ drops below the time-averaged backlog of the non-learning baseline in each environment.  It can also be seen that the moving average backlog $\bar q_t^{MA}$ of the IA-PPO algorithm is less noisy than that of IA-PG, especially for the SH2 and MH2 results. The SH2 and MH2 environments were also the more challenging environments as it took approximately 300,000 timesteps before $\bar q_t^{(MA)}$ of the IA-PPO algorithm was less than that of the MaxWeight/Backpressure policies, wheare is took closer to 100,000 timesteps to accomplish the same in the SH1 and MH1 environments.  These environments had a higher-dimensional state-space compared to the SH1 and MH1 environments. Also, it can be inferred that the effective state-space in which the agent's encountered is much larger as seen by the maximum of $q_t^{(MA)}$ encountered over the experiments length.

\section{Conclusion}
In conclusion, this work introduces a novel intervention-assisted policy gradient approach for enabling Online Deep Reinforcement Learning Controls (ODRLC) in stochastic queuing networks. Our methods, IA-PG and IA-PPO, merge classical control's stability with neural networks' adaptability, showing superior queue stability and network optimization in real-time over traditional methods. Experiments confirm our framework's effectiveness, overcoming unbounded queue challenges and setting a theoretical groundwork for future DRL applications in complex systems.

Future efforts will refine intervention mechanisms, explore scalability, and extend our framework to other domains with similar issues. This research paves the way for integrating traditional control and modern machine learning for advanced system optimization and control.

\bibliographystyle{ACM-Reference-Format}
\bibliography{mobihoc_references}
\appendix

\section{Neural Network Architectures and Policy Parameterization}

For the actor and critic networks, a Multilayer Perceptron (MLP) served as the foundational architecture, featuring two hidden layers, each with a width of 64 neurons. The policy gradient algorithms we use in this work require parameterized stochastic policies. Stochastic policies are essential in the learning process as they enable on-policy exploration when generating trajectories.  These stochastic policy distributions must also be parametric to allow for approximation via neural networks.   However, the selection of suitable stochastic policy distributions is crucial, particularly given that our action sets vary depending on the state. For single-hop scheduling problems, where the action space is small and discrete, we can use a categorical policy architecture and can enforce the valid flow constraints via action masking.    

Conversely,  Multihop network control problems have a high-dimensional discrete action space and complex state-dependent flow constraints that cannot be encoded using traditional stochastic policy parameterizations.   To address these challenges, we developed a novel multinomial stochastic policy parameterization.  This approach effectively incorporates the Multihop flow constraints into a high-dimensional discrete stochastic policy.

To illustrate the challenging constraints for Multihop Multiclass networks consider the network shown in \Cref{fig:MH1}. This network has two class $(K=2)$  and six links  $(M=6)$. For the MH1 network,  network flow decision at time $t$ is represented as a $(6,3)$ matrix $\mathbf A_t$, with $A_{m, k+1, t}$ denoting the capacity allocated to class $k$ on link $m$ and $A_{m,1,t}$ denotes the amount of capacity left unused on link $m$ at step $t$.  For each time step, $\mathbf A_t$ must satisfy the capacity constraints 
\begin{align} \label{eq: link_const}
  &\sum_{k=1}^{K+1}A_{m,k,t} = y_{m,t}  \quad  \forall \quad m\in 1,2,..., M
\end{align}   where $y_{mt}$ is the current capacity of link $m$. Note that these constraints are row-wise constraints over the link flow $\mathbf A_{m,t}$ and this cannot be expressed as individual constraints over each element in $\mathbf A_t$.   
 
 Existing policy gradient algorithms often utilize multivariate normal or beta distributions for high-dimensional control tasks. However,  these distributions cannot be used to enforce row dependant constraints such as \cref{eq: link_const}. Both of these distributions can be used to to enforce element-wise constraints, but cannot enforce the row-wise constraints required for Multihop control problems.  Furthermore, both distributions require continuous-to-discrete conversion for packet allocation, introducing potential discretization errors. To address these limitations, we propose a novel multinomial policy parameterization. This approach effectively captures the high-dimensional nature of the action space while seamlessly enforcing the true constraints dictated by the link-flow constraints.

\subsection{Multinomial Policy Parameterization}
The goal is to design a policy parameterization that allows us to map each state $\mathbf s=(\mathbf q, \mathbf y)$ to a state-dependent distribution $\mathbb P(\cdot|\mathbf s)$ over $\mathcal A(\mathbf y)$, which denotes the set of valid link flows given network link state $\mathbf y$. To this end, we utilize a multinomial distribution parameterization.  A multinomial distribution is a probability distribution that describes the outcomes of multiple trials where each trial has a fixed number of possibilities but only one possibility can occur per trial. In our case, $y_{m,t}$ is the number of ``trials" and $K+1$ is the number of possibilities.  For each link, this is represented as the distribution  $\mathbb P(\mathcal A( y_{m,t})|\mathbf s_t)=Multinomial(\mathbf p_{m,t}, y_{m,t} )$ where $\mathbf p_{m,t}=(p_{m,1,t}, p_{m,2,t}, ..., p_{m,K+1,t})$ and $ p_{m,k+1,t}$ is the probability of choosing class $k$ in one of these independent trials.  We can thus represent the distribution over all valid link flows as the product of $M$ independent multinomial distributions $\mathbb P(\mathcal A(\mathbf y_t)|\mathbf s_t)=$ $\mathbb \prod_{m=1}^{M} P(\mathcal A( y_{m,t})|\mathbf s_t)$.  A sample $\mathbf A_t$ from $\mathbb P(\mathcal A(\mathbf y_t)|\mathbf s_t)$ is guaranteed to satisfies \cref{eq: link_const}. 

To encode this set of $M$  independent multinomial distribution, we use a standard MLP architecture for the input and hidden layers . The output layers have no weights themselves, and are comprised of a $M\times (K+1)$ linear layer and $M$ stacked multinomial layers.  The output of the final linear layer is  $\mathbf{\check p}$, where the $m$th set of $K+1$ elements correspond to the unnormalized logit values for $\mathbf A_{m,t}$. Each set $\mathbf{\check{p}}[m(K+1):]$ is then fed into a multinomial layer which performs the normalization. Unlike the softmax or the Multivariate Gaussian/Beta distribution layers, the multinomial layers require the additional input $y_{m,t}$ as the distribution is not only dependent on the probability parameters.  Recall, for our policy gradient methods we require the ability to sample from $\pi_\theta(\cdot|\mathbf s_t)=\mathbb P(\mathcal A(\mathbf y_t)|\mathbf s_t)$ and to compute log-probabilities $\log(\pi_\theta(\mathbf A_t|\mathbf s_t))=\log \mathbb P(\mathcal A(\mathbf y_t)=\mathbf A_t|\mathbf s_t)$.  From each multinomial layer, the distribution $\mathbb P(\mathcal A( y_{m,t})|\mathbf s_t)$ can be sampled from and combined with the output of the other multinomial layers to yield $\mathbf A_t\sim \mathbb P(\mathcal A(\mathbf y_t)|\mathbf s_t)$.   Given decision $\mathbf A_t$, we can compute the log-probability of each row $\mathbf A_{m,t}$ as:  
\begin{align}
     \log \mathbb P(\mathcal A_{m,t}=\mathbf A_{m,t}|\mathbf s_t) = & \sum_{k=1}^{K+1} A_{m,k,t}\log\left( p_{m,k,t}\right) \\&+\log\left(\Gamma\left(\sum_{k=1}^{K+1}\mathbf A_{m,k,t}+1\right)\right) - \sum_{k=1}^{K+1}\log\left(\Gamma\left( A_{m,k,t}+1\right)\right)  
 \end{align}
where $p_{m,k,t}$ are the corresponding probability obtained by passing $\mathbf s_t$ through neural network and $\Gamma(\cdot)$ is the gamma function. The gradient $\nabla_\theta$ of the log-probability for each multinomial layer is: 
\begin{equation}
\nabla_{\theta} \log\mathbb P(\mathbf A_{m,t};\mathbf s_t,\theta) = \sum_{k=1}^{K+1} A_{m,k,t}\nabla_{\theta}\log p_{m,k,t}    
\end{equation}
as $p_{m,k,t}$ is the only term that depends on the neural network weights $\theta$. Additionally, the log-probability of full decision $\mathbf A_t = (\mathbf A_{m,t}, ..., \mathbf A_{M,t})^T$ is computed by taking the product of the individual log-probabilities: 
\begin{align}
    \log \pi_\theta(\mathbf A_t|\mathbf s_t) = \log \mathbb P(\mathcal A_t=\mathbf A_t|\mathbf s_t,\theta) = \prod_{m=1}^{M} \log \mathbb P(\mathcal A_{m,t}=\mathbf A_{m,t}|\mathbf s_t,\theta)
\end{align} and differentiation follows the same procedure.

\subsection{Action Masking}
The Multihop architecture ensures all samples  satisfy the link capacity constraint  for Multihop problems and the server selection architecture allows us to satisfy the link activation constraint for single server selection problems. However, there are additional constraints for each type of problem that have not been captured by the architecture alone and require \textit{action masking} \cite{huang2022}.  Action masking works by applying some state-dependent masking function $\text{Mask}(\mathbf{\check p}, \mathbf s)=\mathbf p'$ to the logits outputs $\mathbf{\check p}$ of the neural networks prior to either the Multinomial or Softmax operations.  The goal of action masking is to prohibit invalid actions from ever being sampled, thus typically the masking function sets the logits corresponding to invalid actions to a large negative number such that the discrete distribution determined by the normalized mask logits sets the probability of selecting invalid actions to zero.

For the Multihop Multi class control problem we need to ensure that packets of class $k$ are not routed to nodes from which they cannot reach their destination from.  This is not a state-dependent mask and thus the masking function can be precomputed using Dijkstra's algorithm given the topology $\mathcal G$.    For the single server selection problem we add the additional constraint that all decisions must be \textit{work conserving} which in this case means a link cannot be chosen if that link has zero capacity and we cannot choose the idling action if there exists an active link with packets in the corresponding queue.  Thus this masking function is a function of the current state, and must be computed in each time-step.  

These methodologies, combining multinomial policy parameterization with action masking, provide a robust framework for addressing the complex constraints and high-dimensional action spaces inherent in Multihop Multiclass network control problems. This approach allows for efficient exploration and optimization in these complex policy spaces, enhancing the effectiveness and applicability of policy gradient algorithms in network control scenarios.

\section{Algorithm Hyperparameters}
All algorithms employed a learning rate (\(\alpha\)) of \(3e-4\). In single-hop environments, the rollout lengths were set to \(T_e=2048\), whereas in multi-hop environments, the rollout lengths were shorter, \(T_e=512\). Across all algorithms, each training iteration involved 5 update epochs, with the data segmented into 8 minibatches for each update. For the STOP-PPO algorithm \cite{pavse2023}, parameters \(\tau_{\text{warmup}}=1.0e6\) and \(\beta_{\text{STOP}}=1e-6\) were optimized. This optimization was achieved through a parameter sweep in the SH1 environment, selecting the set of hyperparameters that delivered the highest average performance. A PPO clipping coefficient (\(\epsilon\)) of 0.2 was consistently applied across all algorithms. Furthermore, for the average value constrained critic \cite{ma2021}, a moving average step-size (\(\alpha\)) of 0.2 and a bias coefficient (\(\nu\)) of 0.1 were implemented.

\section{Environment Details} \label{sec: env_details}

The details for each environment used for \Cref{sec: experiments} are given below.  Each class table contains the class id $k$,  arrival node, destination node, arrival amounts with their corresponding probabilities.  For example, for class 1 of SH1 all packets arrive to node 1 and leave the network upon arriving to the base-station (BS).  No packets arrive with probability 0.7 and a single packet arrives with probability 0.3.  Each link table has the rows indexed by the corresponding link ID where ID $m$ corresponds to $y_{m,t}$.  The subsequent columns contain the start node, end node, and the possible link states $y_{m,t}$ with their corresponding probabilities.  For example, link 2 goes from node 2 to the base-station and has link state $y_{2,t}$ equal to $[0, 1, 2]$ with probabilities $[0.2, 0.5, 0.3]$ respectively.

\begin{table}[!h]
\begin{minipage}{0.5\linewidth}
    \centering
    \begin{tabular}{|c|c|c|c|c|}
        \hline
        $k$ & Source & Destination & Arrival & Probability \\
        \hline
        1 & 1 & BS & [0, 1] & [0.7, 0.3] \\
        \hline
        2 & 2 & BS & [0, 1] & [0.3, 0.7] \\
        \hline
    \end{tabular}
    \caption{SH1 Class Information}
\end{minipage}%
\begin{minipage}{0.5\linewidth}
    \centering
    \begin{tabular}{|c|c|c|c|}
        \hline
        $m$ & (Start, End) & Capacity & Probability \\
        \hline
        1 & (1, BS) & $[0, 1]$ & $[0.5, 0.5]$ \\
        \hline
        2 & (2, BS) & $[0, 1, 2]$ & $[0.2, 0.5, 0.3]$ \\
        \hline
    \end{tabular}
    \caption{SH1 Link Information}
\end{minipage}
\end{table}
\begin{table}[!h]
\begin{minipage}{0.49\linewidth}
    \centering
    \begin{tabular}{|c|c|c|c|c|}
        \hline
        $k$&  Source&Destination&  Arrival& Probability \\
        \hline
        1 &  1&BS&  [0, 1]&[0.75, 0.25]\\
        \hline
        2 &  2&BS&  [0, 1]&[0.5, 0.5]\\
        \hline
        3 &  3&BS&  [0, 1]&[0.5, 0.5]\\
        \hline
        4 &  4&BS&  [0, 1]&[0.5, 0.5]\\
        \hline
    \end{tabular}
    \caption{SH2 Class Information}
\end{minipage}
\begin{minipage}{0.49\linewidth}
    \centering
    \begin{tabular}{|c|c|c|c|}
        \hline
         $m$&(Start, End)& Capacity & Probability \\
        \hline
         1&(1, BS) & $[0, 1]$ & $[0.3, 0.7]$ \\
        \hline
         2&(2, BS) & $[0, 1, 2]$ & $[0.2, 0.5, 0.3]$ \\
        \hline
         3&(3, BS) & $[0, 1, 2]$ & $[0.1, 0.1, 0.8]$ \\
        \hline
         4&(4, BS) & $[0, 1, 2, 3]$ & $[0.25, 0.25, 0.25, 0.25]$ \\
        \hline
    \end{tabular}
    \caption{SH2 Link Information}
\end{minipage}
\end{table}

\begin{table}[!h]
\begin{minipage}{0.49\linewidth}
    \centering
    \begin{tabular}{|c|c|c|c|c|}
        \hline
        $k$& Source & Destination &  Arrivals&Probability \\
        \hline
        1 & 1 & 4 &  [0, 1]&[0.2, 0.8]\\
        \hline
        2 & 1 & 4 &  [0, 1]&[0.6, 0.4]\\
        \hline
    \end{tabular}
    \caption{MH1 class information}
\end{minipage}
\begin{minipage}{0.5\linewidth}
    \centering
    \begin{tabular}{|c|c|c|c|}
        \hline
         $m$&(Start, End)& Capacity & Probability \\
        \hline
         1&(1,2) & $[0, 1, 2]$ & $[0, 0.5, 0.5]$ \\
        \hline
         2&(1,3) & $[0, 1]$ & $[0.5, 0.5]$ \\
        \hline
         3&(2,3) & $[0, 2]$ & $[0.2, 0.8]$ \\
        \hline
         4&(3,2) & $[0, 1]$ & $[0.2, 0.8]$ \\
        \hline
         5&(2,4) & $[0, 1]$ & $[0.5, 0.5]$ \\
        \hline
         6&(3,4) & $[0, 2]$ & $[0.2, 0.8]$ \\
        \hline
    \end{tabular}
    \caption{MH1 link information}
\end{minipage}  
\end{table}

\begin{table}[!h]
\begin{minipage}{0.49\textwidth}
    \centering
    \begin{tabular}{|c|c|c|c|c|}
        \hline
        $k$ & Source & Destination & Arrival & Probability \\
        \hline
        1 & 1 & 5 & [0, 4] & [0.5, 0.5] \\
        \hline
        2 & 1 & 6 & [0, 3] & [0.0, 1.0] \\
        \hline
        3 & 1 & 7 & [0, 3] & [0.4,0.6] \\
        \hline
        4 & 1 & 8 & [0, 2] & [0.2,0.8] \\
        \hline
    \end{tabular}
    \caption{MH2 Class Information}
\end{minipage}
\begin{minipage}{0.49\textwidth}
    \centering
\begin{tabular}{|c|c|c|c|}
\hline
$m$ & (Start, End) & Capacity & Probability \\
\hline
1 & (1,2) & [0, 2, 4] & [0.2, 0.4, 0.4] \\
\hline
2 & (1,3) & [3, 5] & [0.5, 0.5] \\
\hline
3 & (1,4) & [0, 2, 4] & [0.2, 0.4, 0.4] \\
\hline
4 & (2,5) & [0,3] & [0,1] \\
\hline
5 & (2,6) & [1, 3] & [0.5, 0.5] \\
\hline
6 & (3,6) & [2, 4] & [0.5, 0.5] \\
\hline
7 & (3,7) & [2, 4] & [0.5, 0.5] \\
\hline
8 & (4,7) & [0, 2] & [0.2, 0.8] \\
\hline
9 & (4,8) & [0,3] & [0, 1.0] \\
\hline
10 & (2,3) & [2, 4] & [0.5, 0.5] \\
\hline
11 & (3,2) & [2, 4] & [0.5, 0.5] \\
\hline
12 & (3,4) & [2, 5, 8] & [0.2, 0.4, 0.4] \\
\hline
13 & (4,3) & [2, 5, 8] & [0.2, 0.4, 0.4] \\
\hline
\end{tabular}
\caption{MH2 Link Information}
\end{minipage}    
\end{table}

\end{document}